\documentclass[letterpaper]{article} 
\usepackage{aaai25}  
\usepackage{times}  
\usepackage{helvet}  
\usepackage{courier}  
\usepackage[hyphens]{url}  
\usepackage{graphicx} 
\usepackage{natbib}  
\usepackage{caption} 
\usepackage[switch]{lineno}

\usepackage{todonotes}
\usepackage{amsmath,amsthm, amsxtra, amsfonts, amssymb, amstext, mathtools}
\usepackage{nicefrac}
\usepackage{xspace}
\usepackage[ruled,vlined,linesnumbered]{algorithm2e}\SetArgSty{upshape}
\usepackage{arydshln}
\usepackage{tabularray}
\usepackage{booktabs}
\usepackage{subcaption}

\newtheorem{theorem}{Theorem}
\newtheorem{lemma}[theorem]{Lemma}

\newcommand{\oea}{\mbox{${(1 + 1)}$~EA}\xspace}

\newcommand{\NSGA}{\mbox{NSGA}\nobreakdash-II\xspace}

\DeclareMathOperator{\E}{E}

\renewcommand*{\d}{\,\mathrm{d}}

\newcommand{\R}{\ensuremath{\mathbb{R}}}

\newcommand{\N}{\ensuremath{\mathbb{N}}}
\newcommand{\Z}{\ensuremath{\mathbb{Z}}}

\let\originalleft\left
\let\originalright\right
\renewcommand{\left}{\mathopen{}\mathclose\bgroup\originalleft}
\renewcommand{\right}{\aftergroup\egroup\originalright}

\renewcommand\t{\top
}    

\title{Runtime Analysis for Multi-Objective Evolutionary Algorithms in Unbounded Integer Spaces}

\author{
  Benjamin Doerr\textsuperscript{\rm 1},
  Martin~S. Krejca\textsuperscript{\rm 1},
  Günter Rudolph\textsuperscript{\rm 2}
}

\affiliations{
  \textsuperscript{\rm 1}Laboratoire d'Informatique (LIX), CNRS, École Polytechnique, Institut Polytechnique de Paris, Palaiseau, France\\
  \textsuperscript{\rm 2}Department of Computer Science, TU Dortmund University, Dortmund, Germany\\
  \{first-name.last-name\}@polytechnique.edu,
  guenter.rudolph@tu-dortmund.de
}

\begin{document}

\maketitle

\sloppy{
\begin{abstract}
  Randomized search heuristics have been applied successfully to a plethora of problems.
  This success is complemented by a large body of theoretical results.
  Unfortunately, the vast majority of these results regard problems with binary or continuous decision variables -- the theoretical analysis of randomized search heuristics for unbounded integer domains is almost nonexistent.
  To resolve this shortcoming, we start the runtime analysis of multi-objective evolutionary algorithms, which are among the most successful randomized search heuristics, for unbounded integer search spaces.
  We analyze single- and full-dimensional mutation operators with three different mutation strengths, namely changes by plus/minus one (unit strength), random changes following a law with exponential tails, and random changes following a power-law. The performance guarantees we prove on a recently proposed natural benchmark problem suggest that unit mutation strengths can be slow when the initial solutions are far from the Pareto front. When setting the expected change right (depending on the benchmark parameter and the distance of the initial solutions), the mutation strength with exponential tails yields the best runtime guarantees in our results --  however, with a wrong choice of this expectation, the performance guarantees quickly become highly uninteresting. With power-law mutation, which is an essentially parameter-less mutation operator, we obtain good results uniformly over all problem parameters and starting points.
  We complement our mathematical findings with experimental results that suggest that our bounds are not always tight. Most prominently, our experiments indicate that power-law mutation outperforms the one with exponential tails even when the latter uses a near-optimal parametrization. Hence, we suggest to favor power-law mutation for unknown problems in integer spaces.
\end{abstract}

\section{Introduction}

For more than thirty years, the mathematical runtime analysis of randomized search heuristics has supported the design and analysis of these important algorithms, both in single- and in multi-objective optimization~\cite{NeumannW10,AugerD11,Jansen13,ZhouYQ19,DoerrN20}. While in practice heuristics are successfully employed for all types of decision variables, the mathematical analysis is mostly concentrated on binary or continuous variables.
Discrete spaces with more than two variable values are considered far more infrequently. Theoretical works include runtime analyses for evolutionary algorithms, ant-colony optimizers, and estimation-of-distribution algorithms on categorical variables, e.g., \citet{ScharnowTW04,BaswanaBDFKN09,SudholtT12,DoerrHK12,BenJedidiaDK24,AdakW24}. The theoretical research for cardinal variables was started by \citet{DoerrJS11,DoerrP12,KotzingLW15} with analyses how the \oea optimizes multi-valued linear functions. \citet{DoerrDK18,DoerrDK19} showed that for multi-valued variables larger mutation rates can be advantageous.
Submodular functions with multi-valued discrete domain were studied by \citet{QianSTZ18submodular,QianZTY18}.
We are only aware of two analyses for search spaces consisting of unbounded integer variables.
The first is \cite{Rudolph23}, which is a single-objective analysis of subproblems of a multi-objective problem.
The second is \cite{HarderKLRR24IntegerValued}, which considers a single-objective problem and shows the benefit of larger mutation rates.
We are not aware of any true multi-objective works with multi-valued variables, despite considerable recent theoretical research on multi-objective heuristics \cite{ZhengLD22,DangOSS23aaai,CerfDHKW23,DoNNS23,DinotDHW23,ZhengD24,ZhengLDD24,BianRLQ24,RenBLQ24}.

Since multi-objective optimization is an area where heuristics, in particular, evolutionary algorithms, are intensively used and with great success (see, e.g., the famous \NSGA algorithm~\cite{DebPAM02} with more than $50\,000$ citations on Google Scholar), we start in this work the mathematical runtime analysis for truly multi-objective optimization problems in unbounded integer search spaces. For the benchmark problems proposed by Rudolph~\shortcite{Rudolph23}, we analyze the performance of the \emph{simple evolutionary multi-objective optimizer (SEMO)}~\cite{LaumannsTZWD02} and the \emph{Global SEMO (GSEMO)} \cite{Giel03}, the two most prominent evolutionary algorithms in the runtime analysis of multi-objective evolutionary algorithms. Our objective is to understand, via theoretical means, what are suitable mutation operators for unbounded integer domains. We propose three natural operators, changing variables (i)~by plus or minus one (unit mutation strength), (ii)~by a random value drawn from a symmetric distribution with exponential tails, and (iii)~by a random value drawn from a symmetric power-law distribution.

We conduct an extensive mathematical runtime analysis for these two algorithms with the three mutation operators (with general parameters, where applicable) on the benchmark with general width parameter $a$ of the Pareto front for a general initial solution $x^{(0)}$. Our results, presented in more detail later in this work when all ingredients are made precise, indicate the following. All algorithm variants can compute the full Pareto front of our benchmark problem in reasonable time. The unit mutation strength, not surprisingly can lead to a slow progress towards the Pareto front when the initial solution is far, but it also results in a slow exploration of the Pareto front after having reached it. The performance with the exponential-tail mutation strength depends heavily on the variance parameter~$q$ (which is essentially the reciprocal of the expected absolute change). With an optimal choice of $q$, depending on the benchmark parameter $a$ and the starting solution $x^{(0)}$, this operator leads to the best performance guarantee among our results. However, our runtime guarantees strongly depend on the relation of $q$, $a$, and $x^{(0)}$; hence a suboptimal choice of $q$ can quickly lead to very weak performance guarantees. The power-law mutation strength is a good one-size-fits-all solution. Being an essentially parameter-less operator, it achieves a good performance uniformly over all values of $a$ and $x^{(0)}$, clearly beating the unit mutation strength for larger instances or distances of the starting solution from the Pareto front.

We complement our theoretical results by an empirical analysis, aimed to understand how tight our guarantees are.
We observe that the best-possible parameter regime for the exponential-tail mutation is within the range of our mathematical findings.
However, surprisingly, the exponential-tail mutation is not able to outperform the power-law mutation, even with a near-optimal parametrization of the former.
This suggests that our bounds for the power-law mutation are not tight.
In fact, our experiments indicate that the actual expected runtime for this operator in the considered setting is linear in the size of the Pareto front, whereas our guarantees bound it by a polynomial with an exponent between~$1$ and~$2$, depending on the parametrization of the operator.
We speculate that the actual linear runtime is a result of the complex population dynamics, which, to the best of our knowledge, have not been studied in the detail necessary for an improvement in any theoretical study of the (G)SEMO algorithms.

Overall, our work shows that standard heuristics with appropriate mutations can deal well with certain multi-objective problems with unbounded integer decision variables. We strongly suggest the parameter-less power-law operator, as it has shown the best performance empirically and also has a uniformly good performance guarantee in a wide range of situations in our theoretical findings.

Our proofs can be found in the appendix.

\section{Preliminaries}
\label{sec:preliminaries}

The natural numbers~$\N$ include~$0$.
For $a, b \in \R$, let $[a .. b] = [a, b] \cap \N$, define $[a] \coloneqq [1 .. a]$, and let $\N_{\geq a} = [a, \infty) \cap \N$.

Given $n, d \in \N_{\geq 2}$, we call $f\colon \Z^n \to \R^d$ a \emph{$d$-objective function}, which we aim to minimize.
We call a point $x \in \Z^n$ an \emph{individual}, and~$f(x)$ the \emph{objective value of~$x$}.
For $i \in [n]$ and $j \in [d]$, let~$x_i$ denote the $i$-th component of~$x$, and let~$f_j(x)$ denote the $j$-th component of~$f(x)$.

The objective values of a $d$-objective function~$f$ follow a weak partial order, denoted by~$\preceq$.
For all $u, v \in \R^d$, we say that \emph{$u$ weakly dominates~$v$} ($u \preceq v$) if and only if for all $i \in [d]$, we have $u_i \leq v_i$.
We say that $u$ \emph{strictly} dominates~$v$ ($u \prec v$) if and only if at least one of these inequalities is strict.
We extend this notation to individuals, where a dominance holds if and only if it holds for the respective objective values.

We consider the minimization of~$f$, that is, we are interested in $\preceq$-minimal images.
The set of all objective values that are not strictly dominated, that is, the set $F^* \coloneqq \{f(y) \in \R^d \mid y \in \Z^n \land \nexists x \in \Z^n\colon f(x) \prec f(y)\}$, is the \emph{Pareto front of~$f$}.
Furthermore, the individuals that are mapped to the Pareto front, that is, the set $f^{-1}(F^*)$, is the \emph{Pareto set of~$f$}.

\subsection{Algorithmic Framework}
\label{sec:algorithm}

\begin{algorithm}[t]
  \caption{Algorithmic framework for evolutionary multi-objective minimization of a given $d$-objective function $f\colon \Z^n \to \R^d$.
    The framework requires an initial individual $x^{(0)} \in \Z^n$ and a mutation operator ``mutation''.
    If this operator modifies exactly one position, the resulting algorithm is the SEMO.
    If it modifies each position randomly and independently, the resulting algorithm is the GSEMO.
  }\label{algo:GSEMO}
  $P^{(0)}=\{x^{(0)}\}$\;
  $t \gets 0$\;
  \While{\emph{termination criterion not met}}{%
    \label{line:selectParent}
    choose $x^{(t)}$ from $P^{(t)}$ uniformly at random\;
    $y^{(t)} \gets \mathrm{mutation}(x^{(t)})$\;\label{line:mutation}
    $Q^{(t)} \gets P^{(t)} \smallsetminus \{z \in P^{(t)}\colon f(y^{(t)}) \preceq f(z)\}$\;
    \lIf{$\not\exists z \in Q^{(t)}\colon f(z) \prec f(y^{(t)})$}{%
      $P^{(t + 1)} \gets Q^{(t)} \cup \{y^{(t)}\}$%
    }
    \lElse{%
      $P^{(t + 1)} \gets Q^{(t)}$%
    }
    $t \gets t + 1$\;
  }
\end{algorithm}

We consider the framework of evolutionary multi-objective minimization (Algorithm~\ref{algo:GSEMO}), aimed at finding the Pareto front of a given $d$-objective function.
The algorithm acts iteratively and maintains a multi-set of individuals (the \emph{population}) that are not strictly dominated by any of the so-far found individuals.
In each iteration, the algorithm creates a new individual~$y$ from a random individual from the population by applying  an operation known as \emph{mutation}.
Afterward, all individuals weakly dominated by~$y$ are removed from the population, and~$y$ is added if it is not strictly dominated by any of the remaining individuals in the population.

\paragraph{Mutation.}
We consider \emph{single-} and \emph{full-dimensional} mutation.
Either acts on a \emph{parent} $x \in \Z^n$, requires a distribution~$D$ on~$\Z$ (the \emph{mutation strength}; see also \emph{Runtime Analysis}), and returns an \emph{offspring} $y \in \Z^n$.

Single-dimensional mutation chooses a single component $i \in [n]$ as well as an independent sample $Z \sim D$ and then sets $y_i = x_i + Z$.
All other components remain unchanged, that is, for all $j \in [n] \setminus \{i\}$, it holds that $y_j = x_j$.
Algorithm~\ref{algo:GSEMO} with single-dimensional mutation results in the \emph{SEMO} algorithm \cite{LaumannsTZWD02}.
Full-dimensional mutation does the following independently for each component $i \in [n]$ of~$x$:
With probability~$1/n$, draw an independent sample $Z_i \sim D$, and set $y_i = x_i + Z_i$.
With the remaining probability, set $y_i = x_i$.
Hence, in expectation, exactly one component of~$x$ is changed, while any number of components may be changed.
Algorithm~\ref{algo:GSEMO} with full-dimensional mutation results in the \emph{global SEMO} (GSEMO) \cite{Giel03}.

\paragraph{Runtime.}
The \emph{runtime} of Algorithm~\ref{algo:GSEMO} is the (random) number of function evaluations until the objective values of the population contain the Pareto front.
We do not re-evaluate individuals, but equal individuals are separately evaluated.
That is, the initial individual is evaluated once, and the algorithm evaluates exactly one solution (namely~$y^{(t)}$) each iteration.
Hence, the runtime is~$1$ plus the number of iterations until the Pareto front is covered.

\subsection{Benchmark Problem}
\label{sec:benchmark}

We consider the following bi-objective benchmark function, introduced by Rudolph~\shortcite{Rudolph23}.
Given a parameter $a \in \N$, the function $f\colon \Z^n \to \N^2$ is defined as
\begin{align*}
  x \mapsto
  \begin{pmatrix}
    |x_1 - a| + \sum_{i \in [2 .. n]} |x_i| \\
    |x_1 + a| + \sum_{i \in [2 .. n]} |x_i|
  \end{pmatrix} .
\end{align*}
This function aims at minimizing the distance to two target points, which is the same idea present in similar benchmarks (\textsc{OneMax} and \textsc{OneMinMax}~\cite{GielL10}) that are used as initial problems for related settings.

The Pareto set and front of~$f$ satisfy \cite{Rudolph23}
\begin{align*}
  X^* & = \{ (k, 0, \ldots, 0)^\t \in \Z^n \mid k \in [-a,a]\cap\Z \} \textrm{ and} \\
  F^* & = \{ (k, 2a - k)^\t \in \N^2 \mid k \in [0 .. 2a] \} .
\end{align*}
We note that $|F^*| = 2a + 1$, as this value plays a crucial role in our analyses \emph{(Runtime Analysis)}.

\subsubsection{Useful Properties of the Benchmark Problem}
\label{sec:runTime:usefulProperties}

We study useful properties of~$f$, partially in the context of Algorithm~\ref{algo:GSEMO}.
Throughout this section, we assume that $a \in \N$ and that $n \in \N_{\geq 2}$.
When we consider an instance of Algorithm~\ref{algo:GSEMO}, we allow for \emph{arbitrary} mutation.

Lemma~\ref{lem:alwaysComparableOutside} shows when a solution whose first coordinate is not in $(-a, a)$ is comparable (w.r.t.~$f$) to any other point.
\begin{lemma}
  \label{lem:alwaysComparableOutside}
  Let $x \in \Z^n$ and $y \in \Z_{\geq a} \times \Z^{n - 1}$ such that $f_1(x) \leq f_1(y)$.
  Then $f(x) \preceq f(y)$.

  Similarly, let $x \in \Z^n$ and $y \in \Z_{\leq -a} \times \Z^{n - 1}$ such that $f_2(x) \leq f_2(y)$.
  Then $f(x) \preceq f(y)$.
\end{lemma}

Lemma~\ref{lem:alwaysComparableOutside} implies that the population has at most one solution with its first component at most~$-a$ and at most one with its first component at least~$a$, formalized below.

\begin{lemma}\label{lem:aussen}
  Let $t \in \N$.
  Then $|P^{(t)} \cap (\Z_{\ge a}\times\Z^{n-1})| \le 1$ and $|P^{(t)} \cap (\Z_{\le -a}\times\Z^{n-1})| \le 1$.
\end{lemma}

Lemma~\ref{lem:atMostOneSolutionPerValueInside} shows that the algorithm contains at most one solution per $x_1$-value in $[-a .. a]$.

\begin{lemma}
  \label{lem:atMostOneSolutionPerValueInside}
  Let $t \in \N$, and let $i \in [-a .. a]$.
  Then $|P^{(t)} \cap (\{i\} \times \Z^{n - 1})| \le 1$.
\end{lemma}

The bounds on the population size from Lemmas~\ref{lem:aussen} and~\ref{lem:atMostOneSolutionPerValueInside} imply the following bound on the overall population size.

\begin{lemma}\label{lem:popsize}
  Let $t \in \N$.
  Then $|P^{(t)}| \le 2a+1$.
\end{lemma}

Lemma~\ref{lem:dominanceImpliesL1Norm} shows that the dominance of two points implies an order with respect to the L1-norm of the two points.

\begin{lemma}
  \label{lem:dominanceImpliesL1Norm}
  Let $x, y \in \Z^n$.
  If $f(x) \preceq f(y)$, then $\|x\|_1 \leq \|y\|_1$.
\end{lemma}

Lemma~\ref{lem:dominanceImpliesL1Norm} implies that the minimum L1-norm of the population of the algorithm cannot increase over time.
The following lemma formalizes this property.
It is the main driving force of our theoretical analyses in \emph{Runtime Analysis}.

\begin{lemma}
  \label{lem:noIncreaseInL1Norm}
  Let $t \in \N$, and let $z^* \in P^{(t)}$ be such that $\|z^*\|_1 = \min_{z \in P^{(t)}} \|z\|_1$.
  Moreover, assume for the offspring that $\|y^{(t)}\|_1 > \|z^*\|_1$.
  Then $z^* \in P^{(t + 1)}$.
\end{lemma}

\paragraph{Optimization Dynamics}
The lemmas above show that the dynamics of Algorithm~\ref{algo:GSEMO} on~$f$ are restricted to individuals with their first component in $[-a + 1 .. a - 1]$ as well as at most two individuals with their first component at most~$-a$ or at least~$a$, respectively.
Since Lemma~\ref{lem:dominanceImpliesL1Norm} shows that the L1-norm of such individuals cannot increase, their distance to the Pareto cannot increase either.
Thus, if improvements occur, the population eventually covers the entire Pareto front.

\section{Runtime Analysis}
\label{sec:runTime}

We analyze the expected runtime of Algorithm~\ref{algo:GSEMO} instantiated as the SEMO and as the GSEMO \emph{(Algorithmic Framework)} when optimizing function~$f$ \emph{(Benchmark Problem)}.
We assume for~$f$ implicitly that $a \in \N$ and that $n \in \N_{\geq 2}$.

We consider three different mutation strengths, each characterized by a law over~$\Z$:
the uniform law over $\{-1, 1\}$, a law with an exponential tail, and a power-law.

Although the mutation strength greatly affects the expected runtime of the algorithm, our analyses follow the same general outline.
Each analysis is split into two phases.
The first phase considers the time until the algorithm contains the all-$0$s vector, which is part of the Pareto front~$F^*$ of~$f$.
The second phase considers the remaining time until the objective values of the population contain~$F^*$.
The total runtime bound is then the sum of the bounds of both phases.

For the first phase, we use that the minimum L1-distance of the population to the all-$0$s vector never increases (Lemma~\ref{lem:noIncreaseInL1Norm}).
Formally, we define a potential function that measures this distance, and we utilize tools introduced below for deriving the expected runtime for this phase.

Throughout our analyses, we use that the population consists by Lemma~\ref{lem:popsize} of at most $2a + 1$ individuals.
This results in a probability of at least $1/(2a + 1)$ for choosing a specific individual for mutation and, thus, in an expected time of $2a + 1$ for making such a choice.
In addition, the probability to change exactly one component of a solution is for both algorithms in the order of~$1/n$.
Combining this with the choice for a specific individual yields an overall waiting time of about~$(2a + 1)n$, which all of our results have in common.

The speed of each phase is heavily determined by the mutation strength, leading to different analyses.

\subsection{Mathematical Tools}
\label{sec:runTime:mathematicalTools}

Variable drift theorems translate information about the expected progress of a random process into bounds on expected hitting times. This concept was independently developed by Mitavskiy et~al.~\shortcite{MitavskiyRC09} and Johannsen~\shortcite{Johannsen10}. The following variant is from Doerr et~al.~\shortcite[Theorem~6]{DoerrDY20}.

\begin{theorem}[Discrete variable drift, upper bound] \label{thm:vardrift}
  Let $(X_t)_{t\ge0}$ be a sequence of random variables in $[0..n]$, and let~$T$ be the random variable that denotes the earliest point in time $t \ge 0$ such that $X_t = 0$. Suppose that there exists a monotonically increasing function $h\colon [n]\to\R_0^+$ such that $\E\left[X_t-X_{t+1}\mid X_t\right]\ge h(X_t)$ holds for all $t<T$. Then
  \begin{equation*}
    \E\left[T\mid X_0\right]\le \textstyle\sum\nolimits_{i \in [X_0]}\frac{1}{h(i)}.
  \end{equation*}
\end{theorem}

Additive drift is a simplification of variable drift to the case that the expected progress of a random process is bounded by a constant value.
It dates back to a theorem by He and Yao~\cite{HeY04}.
The following simplified version is from Kötzing and Krejca~\shortcite[Theorem~$7$]{KotzingK19}.

\begin{theorem}[Additive drift, unbounded search space]
  \label{thm:additiveDrift}
  Let $(X_t)_{t \in \N}$ be random variables over~$\R_{\geq 0}$, and let $T = \inf\{t \in \N \mid X_t = 0\}$.
  Furthermore, suppose that there is some value $\delta \in \R_{> 0}$ such that for all $t < T$ holds that $\E[X_t - X_{t + 1} \mid X_t] \geq \delta$.
  Then $\E[T \mid X_0] \leq \frac{X_0}{\delta}$.
\end{theorem}

\subsection{Unit-Step Mutation}
\label{sec:runTime:unaryMutation}

\emph{Unit-step mutation} uses the uniform law over $\{-1, 1\}$.
When modifying the value $x_i \in \Z$ of an individual~$x$, then~$x_i$ is increased by~$1$ with probability~$\frac{1}{2}$ and else decreased by~$1$.

Our main result for unit-step mutation is Theorem~\ref{thm:runTimeSemoGsemo}, showing that the expected runtime of the SEMO and the GSEMO scales linearly in the L1-norm of the initial solution~$x^{(0)}$ (plus~$2a$), and linearly in the dimension~$n$ and the radius~$a$ of the Pareto front of~$f$.
The linear scaling in~$an$ is due to waiting to make progress, as discussed at the start of \emph{Runtime Analysis}.
The linear scaling in~$\|x^{(0)}\|_1$ is due to the first phase, as the mutation changes a component by only~$1$ and needs to cover a distance of~$\|x^{(0)}\|_1$.
Then, since the mutation changes a component only by~$1$, an additional time linear in~$a$ is required for covering the entire Pareto front.

\begin{theorem}
  \label{thm:runTimeSemoGsemo}
  Consider the SEMO or the GSEMO with unit-step mutation optimizing~$f$, given~$x^{(0)}$.
  Let $T = \inf\{t \in \N \mid F^* \subseteq f(P^{(t)})\}$.
  Then $\E[T \mid x^{(0)}] \leq 2n (2a + 1) (\|x^{(0)}\|_1 + 2a)$ for the SEMO, and $\E[T \mid x^{(0)}] \leq 2en (2a + 1) (\|x^{(0)}\|_1 + 2a)$ for the GSEMO.
\end{theorem}

The following lemma bounds the expected time of the first phase.
In each iteration, we have the same probability to improve the solution that is closest in L1-distance to the all-$0$s vector, resulting in a linear bound in the distance~$\|x^{(0)}\|_1$.

\begin{lemma}
  \label{lem:timeToZeroInsideMagicInterval}
  Let $T_1 = \inf\{t \in \N \mid (0, \dots, 0)^\t \in P^{(t)}\}$.
  Then $\E[T_1 \mid x^{(0)}] \le 2n \cdot (2a+1) \cdot \|x^{(0)}\|_1$ for the SEMO, and $\E[T_1 \mid x^{(0)}] \le 2en \cdot (2a+1) \cdot \|x^{(0)}\|_1$ for the GSEMO.
\end{lemma}

The next lemma bounds the expected time of the second phase.
Since the mutation only changes components by~$1$, starting from the all-$0$s vector, the Pareto front is covered in a linear fashion, expanding to either side.
Hence, the resulting time is linear in the size of the Pareto front ($2a + 1$) as well as in the inverse of the probability to choose a specific individual and mutate it correctly (in the order of $an$).

\begin{lemma}
  \label{lem:timeToCompletingParetoFront}
  Let $S \in \N$ be a (possibly random) iteration such that~$P^{(S)}$ contains the individual $(0, \dots, 0)^\t$, and let $T_2 = \inf\{t \in \N \mid t \geq S \land F^* \subseteq f(P^{(t)})\} - S$.
  Then $\E[T_2 \mid S, x^{(0)}] \le 2n \cdot (2a+1) \cdot 2a$ for the SEMO, and $\E[T_2 \mid S, x^{(0)}] \le 2en \cdot (2a+1) \cdot 2a$ for the GSEMO.
\end{lemma}

\subsection{Exponential-Tail Mutation}
\label{sec:runTime:exponentialTails}

\emph{Exponential-tail mutation} utilizes a symmetric law with exponential tails, parameterized by a value $q \in (0, 1)$.
For a random variable~$Z$ following this law, for all $k \in \Z$ holds
\begin{align}
  \Pr[Z=k] = \textstyle\frac{q}{2-q}\,(1-q)^{|k|} .
\end{align}
We call this law a \emph{bilateral geometric law} \cite{Rudolph94ppsn}.

Our main result for exponential-tail mutation is Theorem~\ref{thm:runTimeET}, which is clearly separated into the two phases of our analysis.
Besides the common factor of~$an$ from making progress, as discussed at the beginning of this section, the expected runtime strongly depends on the mutation parameter~$q$.
We discuss the two terms in the following in detail.

\begin{theorem}
  \label{thm:runTimeET}
  Let $c \in (0, 1)$ be constant.
  Consider the SEMO or the GSEMO with exponential-tail mutation with $q \in (0, c)$ optimizing~$f$, given~$x^{(0)}$.
  Let $T = \inf\{t \in \N \mid F^* \subseteq f(P^{(t)})\}$.
  Then there is a sufficiently large constant~$C$ such that for both the SEMO and the GSEMO holds that
  \begin{align*}
    \E[T \mid x^{(0)}]
     & = C\Bigl(an\Bigl(\frac{n}{q} + \|x^{(0)}\|_1 q                               \\
     & \qquad + \max\Bigl\{\frac{\ln(a + 1)}{aq},aq+\ln(a + 1)\Bigr\}\Bigr)\Bigr) .
  \end{align*}
\end{theorem}

Our analysis is based on the following elementary result about the bilateral geometric law.

\begin{lemma}\label{lem:ezz}
  Let $Z$ be a bilateral geometric random variable with parameter $q \in (0,1)$. Let $z \in \N$, and let~$Z_z$ be the random variable defined by $Z_z = Z$, if $0 \le Z \le z$, and $Z_z = 0$ otherwise.  Then $E[Z_z] = \frac{1-q}{q(2-q)} (1 - (1-q)^z (1+zq))$.
  Especially, for all constants $C \in (0, 1)$, there is a constant $K \in \R_{>0}$ such that for all $q \le C$ and all $z \in \N$, we have
  \[\E[Z_z] \ge K \min\{z^2\,q, \tfrac{1}{4q}\}.\]
\end{lemma}

We utilize Lemma~\ref{lem:ezz} in our proofs by estimating by how much a component of an individual chosen for mutation is improved.
Since certain changes can be too large, we consider such progress to be~$0$.
All other values are acceptable.
Hence, we consider overall only a subset of values of the bilateral geometric law, which is well reflected in the lemma.

We now tend to the analysis of Theorem~\ref{thm:runTimeET}.
The first phase is separated into two regimes, depending on how close the minimum L1-distance of the population is to the all-$0$s vector.
If this distance is at least in the order of~$n/q$, the expected distance covered by a successful mutation is in the order of~$1/q$, leading to the term~$\|x^{(0)}\|_1 q$ (in addition to the factor of~$an$ from the waiting time for an improving mutation).
Once the population gets closer than~$n/q$ to the all-$0$s vector, the progress is slowed down and essentially driven by unit changes, resulting in the term~$n/q$.

\begin{lemma}
  \label{lem:timeToZeroET}
  Let $T_1 = \inf\{t \in \N \mid (0, \dots, 0)^\t \in P\}$.
  Then
  \[
    \E[T_1 \mid x^{(0)}] \le \frac{(2a+1)en}{K}\Bigl(\frac{n \pi^2}{6q} + 4 \|x^{(0)}\|_1 q\Bigr),
  \]
  where~$K$ is the constant from Lemma~\ref{lem:ezz}.
\end{lemma}

The expected runtime of the second phase depends on how~$1/q$ compares to~$a$.
We split the runtime into two parts.
The first part concerns covering a subset of the Pareto front that chooses individuals that are roughly~$1/q$ apart, resulting in about~$aq$ intervals of roughly equal size.
If $1/q > 2a + 1$, then we consider the entire Pareto front as a single interval.

The second part concerns covering all intervals.
Since uncovered points in each interval are at most apart by about~$1/q$, we wait for such a rate to be chosen.
This can happen for any interval, leading to independent trials.
We then use a Chernoff-like concentration bound that provides us with a runtime bound that holds with high probability.
Via a restart argument, this bound is turned into an expectation.
The concentration bound yields the logarithmic factors.

\begin{lemma}
  \label{lem:timeToFinishET}
  Assume that at some time $t_2$, the population of the algorithm contains the solution $(0, \dots, 0)$.
  Let $T_2 = \inf\{t \in \N \mid F^* \subseteq f(P^{(t_2+t)})\}$ the additional number of iterations until the Pareto front is computed.
  Then there is a sufficiently large constant $C \in \R_{> 0}$ such that
  \[
    \E[T_2] = C\Bigl(an \max\Bigl\{\frac{\ln(a + 1)}{aq},aq+\ln(a + 1)\Bigr\}\Bigr).
  \]
\end{lemma}

\subsection{Power-Law Mutation}
\label{sec:runTime:powerLaw}

\emph{Power-law mutation} utilizes a symmetric power-law, parameterized by the power-law exponent $\beta \in (1, 2)$ and defined via the Riemann zeta function~$\zeta$.
For a random variable~$Z$ following this law, it holds for all $k \in \Z \setminus \{0\}$ that
\begin{equation*}
  \Pr[Z = k] = |k|^{-\beta} / \bigl(2\,\zeta(\beta)\bigr) .
\end{equation*}

This operator, introduced by \cite{DoerrLMN17}, was shown to be provably beneficial in various settings \cite{FriedrichQW18,CorusOY21tec,AntipovBD22,DangELQ22,DoerrQ23tec,DoerrR23,DoerrKV24,KrejcaW24}.

Our main result for power-law mutation is Theorem~\ref{thm:powerLawRunTimeSemoGsemo}.
Similar to the result for mutation strengths with exponential tails (Theorem~\ref{thm:runTimeET}), the expected runtimes of both phases are well separated.
We explain the details for each phase below.

\begin{theorem}
  \label{thm:powerLawRunTimeSemoGsemo}
  Consider the SEMO or the GSEMO with power-law mutation with constant $\beta \in (1, 2)$ optimizing~$f$, given~$x^{(0)}$.
  Let $T = \inf\{t \in \N \mid F^* \subseteq f(P^{(t)})\}$.
  Then for both the SEMO and the GSEMO, it holds that
  \begin{align*}
     & \E[T \mid x^{(0)}]                                                                                                                     \\
     & \leq (2a + 1) \cdot 2en\zeta(\beta) \left(2^{1/(2 - \beta)} + 2\frac{2 - \beta}{\beta - 1} \|x^{(0)}\|_1^{\beta - 1}\right)            \\
     & \quad + 4 \ln(2) en \frac{\zeta(\beta) (\beta - 1)}{1 - (3/2)^{1 - \beta}} \frac{1}{\left(1 - 2^{1 - \beta}\right)^2} (2a + 1)^{\beta} \\
     & \quad + (2a + 1) \cdot 2en \zeta(\beta) \bigl(\ln(a + 1) + 1\bigr) .
  \end{align*}
\end{theorem}

We make use of the following theorem, which estimates sums of monotone functions via a definite integral.
This is useful, as our analyses involve many possible values for how to improve a specific individual.
These cases lead to sums over of powers of~$-\beta$, as the values follow a power-law.
Integrating such a polynomial is easier than determining the exact value of the discrete sum.
The theorem below shows that we make almost no error when considering the integral.

\begin{theorem}[{\cite[Inequality~(A.12)]{CormenLRS01IntroductionToAlgorithms}}]
  \label{thm:sumsToIntegrals}
  Let $g\colon \R \to \R$ be a monotonically non-increasing function, and let $\alpha, \beta \in \R$ with $\alpha \leq \beta$.
  Then
  \begin{equation*}
    \textstyle\int_{\alpha}^{\beta + 1} g(x) \d x
    \leq \sum\nolimits_{x = \alpha}^{\beta} g(x)
    \leq \int_{\alpha - 1}^{\beta} g(x) \d x .
  \end{equation*}
\end{theorem}

We now consider Theorem~\ref{thm:powerLawRunTimeSemoGsemo}.
The expected runtime of the first phase, besides the factor~$an$, is of order~$\|x^{(0)}\|_1^{\beta - 1}$.
Ignoring the factor~$an$, this is because the algorithm makes an improvement of size~$k$ with probability~$k^{-\beta}$ and has up to about~$\|x^{(t)}\|_1$ choices for an improvement.
This leads to an expected improvement of~$k^{1 - \beta}$ per component of~$\|x^{(t)}\|_1$.
Integrating this expression results in an overall expected improvement of order~$\|x^{(t)}\|_1^{2 - \beta}$.
By the variable drift theorem (Theorem~\ref{thm:vardrift}), estimating the sum via an integral, this translates to an overall runtime of order~$\|x^{(0)}\|_1^{\beta - 1}$.

\begin{lemma}
  \label{lem:powerLawTimeToZeroInsideMagicInterval}
  Let $T_1 = \inf\{t \in \N \mid (0, \dots, 0)^\t \in P^{(t)}\}$.
  Then
  \begin{align*}
     & \E[T_1 \mid x^{(0)}]                                                                                                                    \\
     & \leq (2a + 1) \cdot 2en\zeta(\beta) \left(2^{1/(2 - \beta)} + 2\textstyle\frac{2 - \beta}{\beta - 1} \|x^{(0)}\|_1^{\beta - 1}\right) .
  \end{align*}
\end{lemma}

The second phase advances in several steps.
In each step, the number of points covered on the Pareto front of~$f$ roughly doubles and is spread evenly across the Pareto front.
If the maximum distance of consecutively uncovered points is~$\ell$, the probability to create a new point is at least~$\ell^{1 - \beta}$.
At the beginning, since we assume that we only have a single point on the Pareto front, $\ell$ is in the order of~$a$, and the probability to choose the correct point and mutate it correctly is in the order~$1/(an)$.
Hence, the waiting time for the first step is in the order of~$na^{\beta}$, which dominates the remaining time, as the each additional point on the Pareto front increases the chance of creating a new one.
Since the length of the Pareto front is not a power of~$2$, our result below has two terms.
The first one bounds the time that the doubling procedure above covers at least half of the Pareto front.
The second term bounds the time to cover the remaining points.

\begin{lemma}
  \label{lem:powerLawTimeToCompletingParetoFront}
  Let $S \in \N$ be a (possibly random) iteration such that~$P^{(S)}$ contains the individual $(0, \dots, 0)^\t$, and let $T_2 = \inf\{t \in \N \mid t \geq S \land F^* \subseteq f(P^{(t)})\} - S$.
  Then
  \begin{align*}
     & \E[T_2 \mid S, x^{(0)}]                                                                                                                 \\
     & \quad\le 4 \ln(2) en \frac{\zeta(\beta) (\beta - 1)}{1 - (3/2)^{1 - \beta}} \frac{1}{\left(1 - 2^{1 - \beta}\right)^2} (2a + 1)^{\beta} \\
     & \qquad + (2a + 1) \cdot 2en \zeta(\beta) \bigl(\ln(a + 1) + 1\bigr) .
  \end{align*}
\end{lemma}

\section{Empirical Analysis}
\label{sec:experiments}

\begin{table}\centering
  \begin{tabular}{cr*{3}{r@{$\pm$}r}}
                                  & $1/q$ & \multicolumn{2}{r}{1st hit} & \multicolumn{2}{r}{cover} & \multicolumn{2}{r}{total}                      \\
    \toprule
    U                             &       & 510\,006                    & 25                        & 342\,916                  & 44 & 852\,922 & 11 \\ \hdashline
    \vphantom{\rule{0 pt}{1 em}}E & 5     & 73\,034                     & 8                         & 23\,115                   & 31 & 96\,148  & 10 \\
                                  & 10    & 25\,288                     & 9                         & 18\,346                   & 25 & 43\,634  & 11 \\
                                  & 20    & 9\,028                      & 8                         & 15\,050                   & 22 & 24\,078  & 14 \\
                                  & 50    & 2\,810                      & 11                        & 15\,237                   & 18 & 18\,048  & 16 \\
                                  & 100   & 1\,604                      & 34                        & 18\,401                   & 24 & 20\,004  & 23 \\
                                  & 200   & 1\,613                      & 63                        & 24\,295                   & 20 & 25\,908  & 20 \\
                                  & 500   & 3\,544                      & 104                       & 43\,693                   & 20 & 47\,236  & 23 \\ \hdashline
    \vphantom{\rule{0 pt}{1 em}}P &       & 1\,301                      & 47                        & 14\,263                   & 16 & 15\,565  & 15 \\
    \bottomrule
  \end{tabular}
  \caption{
    Means and standard deviations in percent of 1st hitting time (phase~$1$), Pareto set cover time (phase~$2$), and total runtime (\# evaluations of~$f$) for scenario~$1$ for the GSEMO optimizing~$f$ with the mutation operators: unit-step (U), exponential-tail (E), and power-law (P).
    Column $1/q$ refers to the \emph{step size} of parameter~$q$ of E.
    For~P, we chose $\beta = \frac{3}{2}$.
    The runs were started with $a = 200$ and $x^{(0)} = (0, 100 a)$, with~$50$ independent runs per row for $n=2$.
  }\label{tab:means+sd_n=2}
\end{table}

\begin{figure}
  \includegraphics[width=\columnwidth]{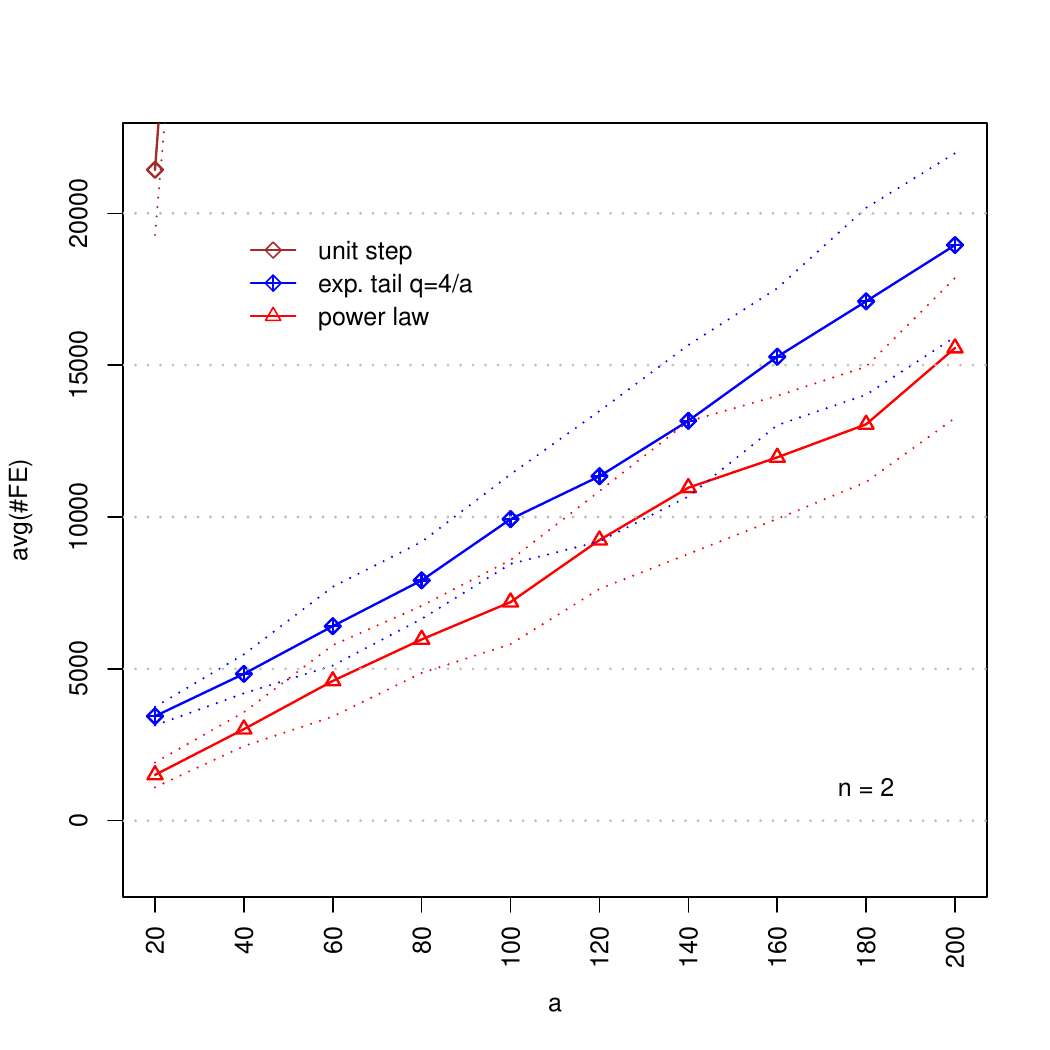}
  \caption{
    The results of scenario~$2$.
    Average evaluations of~$f$ for varying~$a$ for the GSEMO optimizing~$f$ with the mutation operators:
    unit-step (diamonds), exponential-tail (cross diamonds) with $\frac{1}{q} = \frac{a}{4}$, and power-law (triangles) with $\beta = \frac{3}{2}$.
    Each point is based on~$50$ independent runs, with $x^{(0)} = (0, 100 a)$.
    The dotted lines depict the std. deviations.
  }\label{fig:a-plot}
\end{figure}

We aim to assess how far our theoretical upper bounds are from actual, empirical values.
We also aim to understand whether the exponential-tail law actually has a parametrization that is favorable to the power-law mutation, as suggested by our theoretical results.
To this end, we first discuss what runtime behavior we expect, based on our theoretical bounds.
Then we explain our experimental setup and discuss and present our findings.
For the sake of simplicity, we only consider the GSEMO here, as it is the more general algorithm, although benchmark~$f$ is simple enough such that the SEMO would also be sufficient.
We note that when optimizing~$f$, the expected runtime of the GESMO should be worse by a factor of at most~$e$ compared to the SEMO.
Our code is publicly available \cite{DoerrKR24AaaiCodeRepo}.

Similar to our discussion at the start of \emph{Runtime Analysis}, we split the run of the GSEMO into two phases:
The first phase counts the number of function evaluations until the population contains an individual on the Pareto front for the first time.
The second phase counts the remaining number of function evaluations until the Pareto front is covered.

\paragraph{Theoretical considerations.}
In order to compare our theoretical results easily, we assume that the L1-norm of the initial point~$x^{(0)}$ is in the order of the parameter~$a$ of benchmark~$f$.
Furthermore, we assume that the problem size~$n$ is constant, as we consider small values for~$n$ in our experiments, due to the search space being unbounded in any case.
Then the expected runtime for unit-step mutation (Theorem~\ref{thm:runTimeSemoGsemo}) is in the order of~$a^2$.
For exponential-tail mutation with parameter~$q$ (Theorem~\ref{thm:runTimeET}), it is in the order of $a^2 q + a \max\{\frac{\ln(a + 1)}{aq},aq+\ln(a + 1)\}$.
Last, for power-law mutation with power-law exponent~$\beta$ (Theorem~\ref{thm:powerLawRunTimeSemoGsemo}), it is in the order of $a^\beta$.
We see that choosing $q = \frac{1}{a}$ minimizes the maximum expression for the runtime of the exponential-tail mutation, resulting in a runtime bound in the order of $a \ln(a)$.
For this setting, the exponential-tail mutation is fastest and has a quasi-linear runtime.
It is followed by the runtime of power-law mutation, which is a polynomial with degree of~$\beta$, which is between~$1$ and~$2$.
Last, we have the unit-step mutation with a quadratic runtime.

\paragraph{Experimental setup.}
We aim to recreate a setting as described above, placing an emphasis on both phases though, as we do not know how tight our theoretical bounds actually are.
To this end, we choose $n = 2$ and $x^{(0)} = (0, 100 a) \eqqcolon (0, y_0)$.
Furthermore, we choose $\beta = \frac{3}{2}$, which is generally a good choice \cite{DoerrLMN17}.
Our choice for~$n$ is based on all of our results holding for any value of $n \in \N_{\geq 2}$, and a smaller choice of~$n$ lets us run more experiments.
Nonetheless, we note that we consider larger values of~$n$ in the appendix, and the observations are qualitatively the same as they are for $n = 2$, just with an even clearer distinction between the different mutation operators.

We consider two different scenarios:
The first scenario aims to determine a good value~$q$ for the exponential-tail mutation.
To this end, we choose $a = 200$ as well as $\frac{1}{q} \in \{5, 10, 20, 50, 100, 200, 500\}$, which covers a broad, quickly increasing, thus diverse, range of values for~$\frac{1}{q}$.

For each parameter combination of each scenario, we log the number of function evaluations of~$50$ independent runs, always for both phases, even for identical parameter values.

The second scenario observes the runtime behavior of all three mutation operators with respect to~$a$, given a good value for~$q$ determined by scenario~$1$.
We choose $y_0 = 100a$.

\paragraph{First scenario.}
Table~\ref{tab:means+sd_n=2} depicts our results.
For unit step-mutation, we see that it is by far the worst operator for either phase.
For exponential-tail mutation, we see that the choice of~$q$ has an apparently convex impact on the runtime for both phases, with the minimum taken over different values of~$q$.
This conforms mostly with our theoretical insights, where a too small value of~$\frac{1}{q}$ takes needlessly long (like in the case of unit steps), but a too large value of~$\frac{1}{q}$ requires too wait long for actually useful values to appear.
While we expected the choice $\frac{1}{q} = a = 200$ to be best, it is actually $\frac{1}{q} = 50$.
This is mostly due to our theoretical considerations ignoring constant factors and due to the term $\ln(a + 1)$ in the runtime being multiplicative in one case and additive in the other.

More surprisingly, the mean runtime of the power-law mutation for either of the two phases is better than the mean of the exponential-tail distribution for any of our choices of~$1/q$.
This suggests that our runtime estimates are not tight.

\paragraph{Second scenario.}
Given the results from the first scenario, we fix $\frac{1}{q} = \frac{a}{4}$ here.
We now range~$a$ from~$20$ to~$200$ in steps of~$20$.
Our results are depicted in Figure~\ref{fig:a-plot}.

We see that the runtime for exponential-tail mutation is essentially linear, which is what we expected.
However, we also see that the runtime for power-law mutation is roughly linear, which is far better than our theoretical bounds.

One explanation for this discrepancy is that our theoretical analyses always assume that the algorithm's population is maximal, that is, it contains $2a + 1$ individuals.
This assumption seems to be too pessimistic, given preliminary empirical results.
In phase~$1$, if a mutation performs a larger change to an individual, this new individual is likely to strictly dominate multiple solutions in the current population, thus reducing the population size.
This potentially speeds up the first phase.
In phase~$2$, for similar reasons, the initial population can be small and only grows large once almost the entire Pareto front is covered.
Working with smaller populations in between can reduce the runtime of the second phase.

\section{Conclusion}
\label{sec:conclusion}

In this work, we initiated the runtime analysis of multi-objective evolutionary algorithms for unbounded integer spaces.
To this end, we considered variants of the well-known SEMO and GSEMO algorithms.
For each algorithm, we analyzed three different distributions of their mutation strengths.
We derived runtime guarantees for a simple bi-objective problem with Pareto front of size $\Theta(a)$.
Our theoretical results show a complex parameter landscape, depending on the different characteristics of the problem~-- namely, the problem size and~$a$~-- and of the algorithm~-- namely, the mutation strength.

For all reasonable problem parameter choices, the unit-step mutation is the worst update strength, since the progress in each dimension is always bounded by the lowest possible value.
The comparison of the other two mutation strengths is more delicate.
Our theoretical results suggest that the exponential-tail mutation can outperform the power-law mutation if its parameter~$q$ is chosen carefully with respect to the problem parameters.
However, in our experiments the power-law mutation always gave results superior to the exponential-tail mutation using a tuned parameter.
Moreover, our experiments indicate a linear total runtime for both the exponential-tail and the power-law mutation (for certain parameter settings), which is a better runtime behavior for the power-law algorithm than what our upper bounds guarantee.
We speculate this is a consequence of our pessimistic assumption of the algorithms' population size always being maximum.
However, we note that, to the best of our knowledge, this is how essentially any theoretical consideration of the (G)SEMO up to date operates, for example \cite{BianQT18ijcaigeneral,DoerrZ21aaai,DangOSS23aaai}.
Overall, our empirical analysis indicates that either our theoretical guarantees are not tight for the power-law algorithm or that the asymptotic effects are only witnessed for larger problem sizes, noting that our empirical observations seem to hold even more clearly for larger values of~$n$.
In any case, as both our theoretical and empirical results indicate that the power-law mutation has the best expected runtime for a wide range of problem parameters and starting points, and it does not require a careful parameter choice, our general recommendation is to prefer this algorithm for the optimization of problems with unbounded integer variables and no further problem-specific knowledge.

An interesting next step is to analyze whether our theoretical bounds are tight.
We speculate that a more careful study of the algorithms' population dynamics is required.
Theoretical analyses of this level of detail have not been conducted for the (G)SEMO so far.
Hence, more refined analysis techniques than the state of the art seem to be necessary.

Another interesting direction is to prove lower bounds for our considered settings.
Such analyses can shed more light onto certain behavioral aspects of the algorithms, and they require a deeper understanding of the dynamics of the population size.
Thus, they are challenging to derive but can lead to insights that suggest how to improve our upper bounds.

In addition, it would be interesting to analyze other multi-objective evolutionary algorithms, e.g., the very prominent \NSGA \cite{DebPAM02} as well as the NSGA-III \cite{DebJ14}, SPEA2 \cite{ZitzlerLT01}, and the SMS-EMOA \cite{BeumeNE07}.
Moreover, these algorithms have a more complex procedure of updating their population in comparison to that of the (G)SEMO.
This can further prove more challenging.
However, a comparison to other multi-objective algorithms would greatly improve our current theoretical knowledge of the unbounded integer domain.

}

\cleardoublepage
\subsection*{Acknowledgments}
This research benefited from the support of the FMJH Program Gaspard Monge for optimization and operations research and their interactions with data science.

\bibliography{ich_master,alles_ea_master,rest}

\cleardoublepage
\section*{Appendix}
The section titles here share the same names as in the submission.

\subsection*{Useful Properties of the Benchmark Problem}

\begin{proof}[Proof of Lemma~\ref{lem:alwaysComparableOutside}]
  We begin by showing the first claim.
  To this end, note that by the triangle inequality and by $a \geq 0$, we have $|x_1 + a| = |x_1 - a + 2a| \leq |x_1 - a| + 2a$, which is equivalent to $|x_1 + a| - |x_1 - a| \leq 2a$.
  Furthermore, as $y_1 \geq a$ by assumption, it holds that $|y_1 + a| - |y_1 - a| = 2a$.
  Combining these two statements yields $|x_1 + a| - |x_1 - a| \leq |y_1 + a| - |y_1 - a|$.

  Moreover, note that for any $z \in \Z^n$, by the definition of~$f$, it holds that $f_2(z) = f(z)_1 + |z_1 + a| - |z_1 - a|$.
  Using this as well as the assumption $f_1(x) \leq f_1(y)$ and the inequality from the previous paragraph yields that $f_2(x) = f_1(x) + |x_1 + a| - |x_1 - a| \leq f_1(y) + |y_1 + a| - |y_1 - a| = f_2(y)$, proving the first claim.

  For the second claim, note that, by the definition of~$f$, for all $z \in \Z^n$, it holds that $f(-z) = (f_2(z), f_1(z))^\t$.
  Hence, $f_1(-x) \leq f_1(-y)$, and the first claim yields $f(-x) \preceq f(-y)$, that is, $(f_2(x), f_1(x))^\t \preceq (f_2(y), f_1(y))^\t$ and thus $f(x) \preceq f(y)$.
\end{proof}

\begin{proof}[Proof of Lemma~\ref{lem:aussen}]
  By the definition of the algorithms, $P^{(t)}$ only contains incomparable solutions.
  By Lemma~\ref{lem:alwaysComparableOutside}, any two solutions from $\Z_{\ge a}\times\Z^{n-1}$ are comparable, and so are any from $\Z_{\le -a}\times\Z^{n-1}$.
\end{proof}

\begin{proof}[Proof of Lemma~\ref{lem:atMostOneSolutionPerValueInside}]
  For $i \in \{-a, a\}$, the claim follows from Lemma~\ref{lem:aussen}.
  For $i \in [-a + 1 .. a - 1]$ and for all $x, y \in \{i\} \times \Z^{n - 1}$ with $\|x\|_1 \leq \|y\|_1$, it holds that $|x_1 - a| = |y_1 - a| = |i - a|$ as well $|x_1 + a| = |y_1 + a| = |i + a|$ and thus $f_1(x) \leq f_1(y)$ as well as $f_2(x) \leq f_2(y)$.
  Consequently,~$x$ and~$y$ are comparable.
  Since, by the definition of the algorithms, $P^{(t)}$ only contains incomparable solutions, it follows that~$P^{(t)}$ contains at most one solution from $\{x, y\}$.
\end{proof}

\begin{proof}[Proof of Lemma~\ref{lem:popsize}]
  By Lemma~\ref{lem:aussen}, $P^{(t)}$ contains at most one solution with $x_1$-value at most~$-a$ and at most one solution with $x_1$-value at least~$a$.
  Furthermore, by Lemma~\ref{lem:atMostOneSolutionPerValueInside}, $P^{(t)}$ contains for each value $i \in [-a + 1 .. a - 1]$ at most one solution with $x_1$-value~$i$.
  This results in at most $2 + a - 1 - (-a + 1) + 1 = 2a + 1$ solutions.
\end{proof}

\begin{proof}[Proof of Lemma~\ref{lem:dominanceImpliesL1Norm}]
  We prove this statement by contradiction.
  Hence, assume that $\|x\|_1 > \|y\|_1$.
  By the definition of~$f$ and the assumption $f(x) \preceq f(y)$, it follows that
  \begin{align}
    \label{eq:dominanceImpliesL1Norm:xLessThany}
     & |x_1 - a| + \|x\|_1 \leq |y_1 - a| + \|y\|_1 - |y_1| + |x_1| \textrm{ and} \\
    \notag
     & |x_1 + a| + \|x\|_1 \leq |y_1 + a| + \|y\|_1 - |y_1| + |x_1| .
  \end{align}
  Furthermore, by the assumption $\|x\|_1 > \|y\|_1$, it follows that
  \begin{align}
    \label{eq:dominanceImpliesL1Norm:yLessThanx}
     & |x_1 - a| + \|y\|_1 < |x_1 - a| + \|x\|_1 \textrm{ and} \\
    \notag
     & |x_1 + a| + \|y\|_1 < |x_1 + a| + \|x\|_1 .
  \end{align}
  Combining equations~\eqref{eq:dominanceImpliesL1Norm:xLessThany} and~\eqref{eq:dominanceImpliesL1Norm:yLessThanx}, canceling $\|y\|_1$, and rearranging some terms yields
  \begin{align*}
     & |x_1 - a| + |y_1| < |y_1 - a| + |x_1| \textrm{ and} \\
     & |x_1 + a| + |y_1| < |y_1 + a| + |x_1| .
  \end{align*}
  We show that this system of inequalities is infeasible, resulting in the contradiction that concludes the proof.
  Recall that $a \geq 0$.
  First, if $x_1 \leq 0$, then the first inequality simplifies to $|y_1| + a < |y_1 - a|$, which is a contradiction, since, by the triangle inequality, it holds that $|y_1 - a| \leq |y_1| + a$.
  Last, if $x_1 \geq 0$, then the second inequality simplifies to $|y_1| + a < |y_1 + a|$, which is also a contradiction, again by the triangle inequality, concluding the proof.
\end{proof}

\begin{proof}[Proof of Lemma~\ref{lem:noIncreaseInL1Norm}]
  The contraposition of Lemma~\ref{lem:dominanceImpliesL1Norm} implies that $f(y^{(t)}) \npreceq f(z^*)$.
  Thus, by the definition of Algorithm~\ref{algo:GSEMO}, it follows that $z^* \in P^{(t + 1)}$.
\end{proof}

\subsection*{Unit-Step Mutation}

\begin{proof}[Proof of Lemma~\ref{lem:timeToZeroInsideMagicInterval}]
  For both algorithms, we aim to apply the additive drift theorem (Theorem~\ref{thm:additiveDrift}) to the process that considers the smallest L1-distance of an individual in the population to the all-$0$s vector~$O$, that is, we consider $(X_t)_{t \in \N} \coloneqq (\min_{x \in P^{(t)}} \|x\|_1)_{t \in \N}$.
  Note that $X_{T_1} = 0$.
  We consider the natural filtration of $(P^{(t)})_{t \in \N}$.
  Note that~$X$ is adapted to this filtration.
  Furthermore, for the following, let $t \in \N$ and assume that $t < T_1$ is true.

  For either algorithm, by Lemma~\ref{lem:noIncreaseInL1Norm}, $X_t$ cannot increase.
  Hence, it suffices to consider the cases in which~$X_t$ decreases.
  For~$X_t$ to decrease, it is sufficient that the algorithm chooses the individual that is closest to~$O$ and changes exactly one of the values that is not~$0$ yet.
  For both algorithms, by Lemma~\ref{lem:popsize}, choosing the correct individual has a probability of at least $1/(2a + 1)$.
  Still for both algorithms, the probability that the mutation chooses one of the positions that need to be improved and chooses the correct direction for improvement is at least $1/(2n)$.
  For the GSEMO only, we furthermore consider the event that the mutation does not change any of the other probabilities, which occurs with probability at least $(1 - 1/n)^{n - 1} \geq 1/e$.
  All of these probabilities are mutually independent.
  Hence, for the SEMO, we get $\E[X_t - X_{t + 1} \mid P^{(t)}] \geq 1/((2a + 1) \cdot 2n)$, and for the GSEMO $\E[X_t - X_{t + 1} \mid P^{(t)}] \geq 1/((2a + 1) \cdot 2en)$.
  Applying the additive drift theorem, noting that $x^{(0)} \in P^{(0)}$, which implies that $X_0 = \|x^{(0)}\|_1$, concludes the proof.
\end{proof}

\begin{proof}[Proof of Lemma~\ref{lem:timeToCompletingParetoFront}]
  For either algorithm, we aim to apply the additive drift theorem (Theorem~\ref{thm:additiveDrift}) to the random process that measures the difference of the size of the current Pareto front to the size of the maximum-cardinality Pareto front~$F^*$ of~$f$.
  In other words, we consider $(X_t)_{t \in \N} \coloneqq (2a + 1 - |f(P^{(S + t)} \cap F^*|)_{t \in \N}$, noting that $X_0 \leq 2a$, as~$P^{(S)}$ contains at least the all-$0$s vector, which is in~$F^*$.
  Furthermore, note that $X_{T_2} = 0$.
  We consider the filtration that is the smallest possible containing the natural filtration of $(P^{(S + t)})_{t \in \N}$, of~$S$, and of~$x^{(0)}$, noting that~$X$ is adapted to this filtration.
  Last, in the following, let $t \in \N$, and assume that $S + t < T_2$ is true.

  Note that~$X$ cannot increase, as points on the global Pareto front are never removed, due to the definition of the algorithms.

  For either algorithm, since we assume that $X_t \geq 1$, due to our assumption on~$t$, there is at least one point in~$F^*$ that is not in~$P^{(t)}$.
  We consider one such point~$p$ that is closest to a point in $P^{(t)} \cap F^*$.
  Note that such a point~$p$ is in distance~$1$ in the first component, as all other components are~$0$.
  In order for either algorithm to add this point to its current population, the algorithm needs to pick an individual from $P^{(t)} \cap F^*$ closest to~$p$, the independent probability of which is at least $1/(2a + 1)$, due to Lemma~\ref{lem:popsize}.
  Afterward, the mutation needs to choose to change the first position in the correct way.
  For either algorithm, the independent probability to do so is $1/(2n)$.
  Last, only for the GSEMO, no other position must be changed during mutation, which occurs independently with probability $(1 - 1/n)^{n - 1} \geq 1/e$.
  Overall, for the SEMO, we get $\E[X_t - X_{t + 1} \mid P^{(S + t)}, S, x^{(0)}] \geq 1/\bigl((2a + 1) \cdot 2n\bigr)$, and for the GSEMO $\E[X_t - X_{t + 1} \mid P^{(S + t)}, S, x^{(0)}] \geq 1/\bigl((2a + 1) \cdot 2en\bigr)$.
  Applying the additive drift theorem, using the bound on~$X_0$ given at the beginning of the proof, and taking the conditional expected value of the result with respect to~$S$ and~$x^{(0)}$ completes the proof.
\end{proof}

\begin{proof}[Proof of Theorem~\ref{thm:runTimeSemoGsemo}]
  For either algorithm, let~$T_1$ and~$T_2$ be defined as in Lemmas~\ref{lem:timeToZeroInsideMagicInterval} and~\ref{lem:timeToCompletingParetoFront}, respectively, letting~$S$ from Lemma~\ref{lem:timeToCompletingParetoFront} be~$T_2$.
  Then it follows that $T = T_1 + T_2$.
  By linearity of expectation and by the law of total probability, we get $\E[T \mid x^{(0)}] = \E[T_1 \mid x^{(0)}] + \E[\E[T_2 \mid T_2, x^{(0)}] \mid x^{(0)}]$.
  The result then follows by Lemmas~\ref{lem:timeToZeroInsideMagicInterval} and~\ref{lem:timeToCompletingParetoFront}.
\end{proof}

\subsection*{Exponential-Tail Mutation}

Our analysis makes use of the following lemma.

\begin{lemma}
  \label{lem:specialGeometricSeries}
  For $q \in (0, 1)$ and $z \in \N$, it holds that
  \begin{align*}
    \sum_{k=0}^z k\,p^{k-1}        & = \frac{1-p^z\,(1+z\,(1-p))}{(1-p)^2} \textrm{ and}              \\
    \sum_{k = 0}^{z} (k + 1) p^{k} & = \frac{1 - p^{z + 1} \bigl(2 + z(1 - p) - p\bigr)}{(1 - p)^2} .
  \end{align*}
\end{lemma}

\begin{proof}[Proof of Lemma~\ref{lem:specialGeometricSeries}]
  As well known, $\sum_{k=0}^z p^k = \frac{1-p^{z+1}}{1-p}$. Differentiation w.r.t.\ $p$ on both sides leads to
  \begin{align*}
    \sum_{k=0}^z k\,p^{k-1} & = \frac{1-p^{z+1}-(z+1)\,p^z\,(1-p)}{(1-p)^2} \\
                            & = \frac{1-p^z\,(1+z\,(1-p))}{(1-p)^2},
  \end{align*}
  proving the first equation.

  For the second equation, we use the first one and get
  \begin{align*}
     & \sum_{k = 0}^{z} (k + 1) p^{k}
    = p \sum_{k = 0}^{z} k p^{k - 1} + \sum_{k = 0}^{z} p^k                                              \\
     & = p \frac{1 - p^z \bigl(1 + z (1-p)\bigr)}{(1 - p)^2} + \frac{1 - p^{z + 1}}{1 - p}               \\
     & = \frac{p \bigl(1 - p^z \bigl(1 + z (1-p)\bigr) - 1 + p^{z + 1}\bigr) + 1 - p^{z + 1}}{(1 - p)^2} \\
     & = \frac{1 - p^{z + 1} \bigl(2 + z(1 - p) - p\bigr)}{(1 - p)^2} ,
  \end{align*}
  proving the second equation.
\end{proof}

\begin{proof}[Proof of Lemma~\ref{lem:ezz}]
  Let $p=1-q$ temporarily to ease notation.
  Using the first equation from Lemma~\ref{lem:specialGeometricSeries} and multiplying with $p\,(1-p)/(1+p)$ yields
  \begin{align*}
    E[Z_z] & = \sum_{k=0}^z k\,p_k = \frac{1-p}{1+p}\,\sum_{k=0}^z k\,p^k \\
           & = \frac{p\,(1-p^z\,(1+z\,(1-p)))}{(1-p)\,(1+p)}.
  \end{align*}
  Replacing $p=1-q$ finally delivers the desired result for $E[Z_z]$.

  As for the bounds, first note that $\frac{1-q}{2-q}$ is monotonically decreasing on $q\in[0,1]$ so that $\frac{1-q}{2-q}\ge\frac{1 - C}{2 - C}$ for $q\le C$. Using the (strong) Weierstrass product inequality we get $(1-q)^z\le 1-zq + \frac{1}{2}\,(zq)^2$ and finally
  \begin{align*}
    1-(1-q)^z\,(1+zq) & \ge 1-(1-zq + \tfrac{1}{2}(zq)^2)\,(1+zq) \\
                      & = \tfrac{1}{2}(zq)^2\,(1-zq).
  \end{align*}
  Assume $\min\{z^2\,q, \tfrac{1}{4q}\} = z^2\,q$ so that  $z^2\,q \le\frac{1}{4q} \Leftrightarrow 1-zq\ge\frac{1}{2}$. Insertion of all bounds yields
  \begin{align*}
    E[Z_z] & = \frac{1-q}{q(2-q)} (1 - (1-q)^z (1+zq))                                     \\
           & \ge \frac{1}{q}\cdot\frac{1 - C}{2 - C}\cdot\frac{1}{2}(zq)^2\cdot\frac{1}{2} \\
           & = \frac{1}{4}\cdot\frac{1 - C}{2 - C}\,z^2\,q
  \end{align*}
  with $K= \frac{1}{4}\cdot\frac{1 - C}{2 - C}$.

  Now assume $\min\{z^2\,q, \tfrac{1}{4q}\} = \tfrac{1}{4q}$, so that $z\ge\frac{1}{2\,q}$. In order to get the bound $\ge K\,\frac{1}{4\,q} \geq K\,\frac{1}{4\,C}$ we must show that
  $b(z,q) \coloneqq 1-(1-q)^z\,(1+q\,z)$ is bounded from below by a constant independent from $q$ and $z$. Note that, regardless how we choose $q\le C$, the value of $z$ must be at least $\frac{1}{2\,q}$. If we fix $q$ and increase $z$ then $b(z+1,q) \ge b(z,q)$ since $b(z+1,q)-b(z,q) = (1-q)^z\cdot q^2\,(1+z) > 0$. Thus, with $z=\frac{1}{2\,q}$ we choose the lowest admissible value for $z$ and hence also the smallest value of $b(z,q)$, namely $b(\frac{1}{2\,q},q) = 1-\frac{3}{2}\,(1-q)^{1/(2\,q)}$, which is monotonically increasing in $q$ as can be seen from the derivative. As a consequence, we get the smallest value for the limit $q\to 0$: $\lim_{q\to 0} b(1/(2\,q),q) = 1-\frac{3}{2}\cdot\frac{1}{\sqrt{e}} \ge 0.09 = \frac{9}{100}$. Putting all together, we obtain
  \begin{align*}
    E[Z_z] \ge \frac{1}{q}\cdot\frac{1 - C}{2 - C}\cdot\frac{9}{100},
  \end{align*}
  where the right-hand side is constant.
\end{proof}

\begin{proof}[Proof of Lemma~\ref{lem:timeToZeroET}]
  For any population $P$, let $d(P) = \min\{\|x\|_1 \mid x \in P\}$. For any iteration~$t$, let $P_t$ be the population of the SEMO or GSEMO at the start of some iteration~$t$. In the following, we analyze the effect of one iteration. Hence fix any time $t$ and any outcome of the population $P_t$ (hence we condition on $P_t$ in the following).

  We estimate the expected progress (the \emph{drift}) $\E[d(P_t) - d(P_{t+1})]$. Let $x \in P_t$ with $\|x\|_1$ minimal. By Lemma~\ref{lem:noIncreaseInL1Norm} we know that if $x \notin P_{t+1}$, then in the $t$-th iteration an individual $y$ was generated with $\|y\|_1 \le \|x\|_1$. Consequently, we have $d(P_{t+1}) \le d(P_t)$ with probability one. By Lemma~\ref{lem:popsize}, we have $|P_t| \le 2a+1$. Consequently, with probability at least $\frac{1}{2a+1}$, the parent chosen in this iteration equals~$x$. Conditional also on this event, we estimate the drift by regarding separately two cases. To this aim, let $K$ be the constant which exists according to Lemma~\ref{lem:ezz}.

  Case 1: If there is an $i \in [1..n]$ such that $|x_i| > \frac{1}{2q}$, then with probability at least $\frac 1n$ (SEMO) or $\frac 1n (1-\frac 1n)^{n-1}\ge \frac 1 {en}$ (GSEMO) exactly the $i$-th component of $x$ is changed in the mutation. If this happens, by Lemma~\ref{lem:ezz}, the resulting offspring $y$ satisfies $\E[\|y\|_1] \le \|x\|_1 - K \frac{1}{4q}$.
  Consequently, in this first case we have $\E[\|y\|_1] \le \|x\|_1 - \frac{K}{en}\frac{1}{4q}$  regardless of whether the SEMO or GSEMO is used.

  Case 2:
  If $|x_i| \le \frac{1}{2q}$ for all $i \in [1..n]$, then we argue as follows. We only regard the progress made from mutating a single position. If this position is $i \in [1..n]$, then by Lemma~\ref{lem:ezz} again, the expected progress $\|y\|_1 - \|x\|_1$ is at least $K |x_i|^2 q$. The probability that exactly the $i$-th position is mutated, is again  at least $\frac 1n$ (SEMO) or $\frac 1n (1-\frac 1n)^{n-1}\ge \frac 1 {en}$ (GSEMO). Consequently, we have $E[\|y\|_1] \le \|x\|_1 - \frac{Kq}{en} \sum_{i=1}^n |x_i|^2$. By the classic inequality relating the arithmetic and the quadratic mean (a special case of the Cauchy--Schwarz inequality), we have $\frac{Kq}{en} \sum_{i=1}^n |x_i|^2 \ge \frac{Kq}{en^2} (\|x\|_1)^2$.

  Putting all together, including the probability of selecting $x$ as parent, we have
  \begin{align*}
    \E[d(P_t) - d(P_{t+1})] & \ge \frac{K}{(2a+1)en} \min\left\{\frac{1}{4q}, (\|x\|_1)^2 \frac{q}{n}\right\} \\
                            & \coloneqq h(P_t).
  \end{align*}
  Since $T_1$ is the first time that $d(P_t) = 0$, we can apply the variable drift theorem (Theorem~\ref{thm:vardrift}) to the process $d(P_t)$ and obtain
  \begin{align*}
    E[T_1] & = \E[\min\{t \mid d(P_t) = 0\}]                                                                                                                     \\
           & \le \sum_{i=1}^{\|x^{(0)}\|_1} \frac{1}{h(i)}                                                                                                       \\
           & =\frac{(2a+1)en}{K}\left(\sum_{i=1}^{\lfloor\sqrt{n}/(2q)\rfloor} \frac{n}{qi^2} + \sum_{i=\lfloor\sqrt{n}/(2q)\rfloor+1}^{\|x^{(0)}\|_1} 4q\right) \\
           & \le \frac{(2a+1)en}{K}\left(\frac{n \pi^2}{6q} + (\|x^{(0)}\|_1 - \lfloor\sqrt{n}/(2q)\rfloor) \cdot 4q\right)                                      \\
           & \le \frac{(2a+1)en}{K}\left(\frac{n \pi^2}{6q} + 4\|x^{(0)}\|_1 q\right),
  \end{align*}
  where we used that $\sum_{i=1}^\infty \frac 1 {i^2} = \frac{\pi^2}{6}$.
\end{proof}

\begin{proof}[Proof of Lemma~\ref{lem:timeToFinishET}]
  Let $Q = \min\{\lfloor \frac 1q \rfloor,a\}$. Let $m = \lceil \frac aQ \rceil$. For $i \in [1..m]$, let $A_i = [(i-1)Q+1..iQ] \cap [1..a]$ and $A_{-i} = -A_i$. Let $A_0 = \{0\}$. Then $\bigcup_{i=-m}^m A_i = [-a..a]$. For $i \in [-m..m]$, let $\overline A_i = \{(x_1, 0, \dots, 0) \in \R^n \mid x_1 \in A_i\}$. Then $f(\bigcup_{i=-m}^m \overline A_i) = F^*$.

  Our proof strategy is to first analyze the time until for each $i \in [-m..m]$ there is at least one $x \in P^{(t)}$ such that $x \in \overline A_i$ and then analyze the time until all $\overline A_i$ are contained in the population. We note that all elements of the $\overline A_i$ are Pareto optima, and the unique ones with their objective value, so once such a solution is contained in the population of the (G)SEMO, it will stay there forever.

  For the first part, consider $i, j \in [-m..m]$ and a time $t\ge t_2$ such that $|i-j|=1$, $P^{(t)} \cap \overline A_i \neq \emptyset$, and $P^{(t)} \cap \overline A_j = \emptyset$. Note that $j \neq 0$. We estimate the probability that $P^{(t+1)} \cap \overline A_j \neq \emptyset$. Let $x \in P^{(t)} \cap \overline A_i$. By symmetry, we can assume, without loss of generality, that $j = i+1$. By construction, there is a $k \in [1..Q]$, determined by the position of $x$ in $\overline A_i$, such that $A_j = [x_1+k..x_1+Q-1]$. Consequently, if $x$ is selected as parent for the mutation operation, if only the first component of $x$ is modified, and this by increasing it by between $k$ and $k+Q-1$, then the offspring $y$ is contained in $\overline A_j$. By Lemma~\ref{lem:popsize}, the probability of this event is at least $\frac{1}{2a+1} \frac 1{en} Q L$, where $L$ is such that a bilateral geometric random variable $Z$ with parameter $q$ satisfies $\Pr[Z=i] \ge L$ for all $i \in [k..k+Q-1]$. Recalling that $q \le c$, we note that $\Pr[Z=i] = \frac{q}{2-q}(1-q)^i \ge \frac{q}{2-q}(1-q)^{2Q-1} \ge \frac{q(1-q)}{2-q}((1-q)^{(1/q)-1})^2 \ge \frac{q(1-q)}{2-q}e^{-2} \ge q(1 - c)/((2 - c)e^2)$ for all $i$ considered, so $L = q(1 - c)/((2 - c)e^2)$ is a feasible choice (note that here we used that $(1-q)^{(1/q) - 1} \ge \frac 1e$ for all $q \in (0,1)$).
  Consequently, the expected time until $P^{(t+1)} \cap \overline A_j \neq \emptyset$ is the reciprocal of this probability, that is, $\frac{2 - c}{1 - c} (2a+1) e^3 n \frac{1}{Qq}$.
  Since we start with $P^{(t_2)} \cap \overline A_0 \neq \emptyset$, using the above argument $2m$ times (for suitable values of $i$ and $j$), we see that the expected $t'$ to have $P^{(t_2+t')} \cap \overline A_i \neq \emptyset$ for all $i \in [-m..m]$ is $2m \frac{2 - c}{1 - c} (2a+1) e^3 n \frac1{Qq}$.

  We now analyze how the population fills up the blocks $\overline A_i$, $i \in [-m..m]$. Note that we can assume that $Q \ge 2$ as otherwise we are done already. We first regard an arbitrary fixed block $\overline A_i$, $i \neq 0$. Assume that at some time $t$, this block contains at least one individual from $P^{(t)}$. Fixing such an initial situation, we analyze the time $T'_i$ until the population $P^{(t+T'_i)}$ contains $\overline A_i$. Let $x$ be some individual in the current population and in $\overline A_i$ and let $y$ be an element of $\overline A_i$ not contained in the current population. For both the SEMO and the GSEMO, the probability that a mutation operation with $x$ as parent mutates $x$ into $y$ is at least $\frac{1}{en} \Pr[Z = y_1-x_1] = \frac{1}{en} \frac{q}{2-q} (1-q)^{|y_1-x_1|} \ge \frac{1}{en} \frac{q}{2-q} (1-q)^{(1/q)-1} \ge \frac{1}{2 e^2 n} q$; here $Z$ denotes a bilateral geometric random variable with parameter $q$ and the last estimate uses again that $(1-q)^{(1/q) - 1} \ge \frac 1e$ for all $q \in (0,1)$.
  As this bound is independent of $x$ and $y$, a simple union bound over the $k$ choices of $x$ and the $Q-k$ choices of $y$, with the lower bound of $\frac{1}{2a+1}$ for picking a particular $x$ as parent, yields that with probability at least
  \[
    p_k \coloneqq k(Q-k) \frac{1}{2a+1} \frac{1}{2e^2 n} q,
  \]
  this iteration increases the number of individuals in $\overline A_i$ from $k$ to $k+1$. This estimate is valid for any state of the population as long as there are exactly $k$ individuals in $\overline A_i$. Consequently, the time to go from $k$ to $k+1$ individuals is stochastically dominated (see, e.g., \cite{Doerr19tcs} for some background on stochastic domination) by a geometric random variable with success rate $p_k$, and the time $T'_i$ to go from our arbitrary initial state with at least one individual in $\overline A_i$ to a state with $\overline A_i$ fully contained in the population is stochastically dominated by the sum of independent geometric random variables $X_k$ with success rates $p_k$, $k = 1, \dots, Q-1$, that is, $T'_i \preceq \sum_{k=1}^{Q-1} X_k$.

  We note that for $k \in [1..\lfloor Q/2 \rfloor]$, we have $p_k \ge kQ \frac{1}{2a+1} \frac{1}{4e^2 n} q \eqqcolon p'_k$, whereas for $k \ge Q/2$, we have $p_k \ge Q(Q-k) \frac{1}{2a+1} \frac{1}{4e^2 n} q = p'_{Q-k}$. Let $Y^{(1)}, Y^{(2)}$ be two independent random variables, each being the independent sum of $Q$ geometric random variables with success rates $p'_1, \dots, p'_Q$. Then $T'_i \preceq Y^{(1)} + Y^{(2)}$. We note that the $Y^{(j)}, j = 1,2$, fulfill the assumptions of the Chernoff bound for geometric random variables proven in~\cite[Lemma~4]{DoerrD18} (which is Theorem~1.10.35 in~\cite{Doerr20bookchapter}). Consequently,
  \[\Pr\left[Y^{(j)} \ge (1+\delta) \frac{(2a+1) 4e^2 n}{Q^2 q} Q \ln Q \right] \le Q^{-\delta}
  \]
  for all $j = 1,2$ and $\delta > 0$.
  Using a simple union bound, we derive
  \[
    \Pr\left[T'_i \ge (1+\delta) 8 (2a+1) e^2 n \frac{\ln Q}{Qq}\right] \le 2 Q^{-\delta}.
  \]

  Let $\delta =  \log_Q(8m) = \ln(8m)/\ln(Q)$. Then a union bound over the $T'_i, i \in [-m..m] \setminus \{0\}$ shows that with probability at least $\frac 12$, all $\overline A_i$ are contained in the population after $(1+\delta) 8 (2a+1) e^2 n \frac{\ln Q}{Qq}$ time steps. Repeating this argument, we see that after an expected number of $(1+\delta) 16 (2a+1) e^2 n \frac{\ln Q}{Qq}$ iterations, all $\overline A_i$ are contained in the population, that is, the population contains the Pareto set.

  Putting the two phases together, we see that
  \begin{align*}
    \E[T_2] & \le 2m \frac{2 - c}{1 - c} (2a+1) e^3 n \frac1{Qq}                         \\
            & \qquad + (1+\delta) 16 (2a+1) e^2 n \frac{\ln Q}{Qq}                       \\
            & = C\left(\frac{(2a+1)n}{Qq} (m+ \log(Q))\right)                            \\
            & = C\left(an \max\left\{\frac{\ln(a + 1)}{aq},aq+\ln(a + 1)\right\}\right),
  \end{align*}
  where the last estimate follows from noting, via a case distinction, that $\frac{m+ \log(Q)}{Qq} \leq C(\max\{\frac{\ln(a + 1)}{aq},aq+\ln(a + 1)\})$.
\end{proof}

\begin{proof}[Proof of Theorem~\ref{thm:runTimeET}]
  For either algorithm, let~$T_1$ and~$T_2$ be defined as in Lemmas~\ref{lem:timeToZeroET} and~\ref{lem:timeToFinishET}, respectively, letting~$t_2$ from Lemma~\ref{lem:timeToFinishET} be~$T_1$.
  Then it follows that $T = T_1 + T_2$.
  By linearity of expectation and by the law of total probability, we see that $\E[T \mid x^{(0)}] = \E[T_1 \mid x^{(0)}] + \E[\E[T_2 \mid t_2, x^{(0)}] \mid x^{(0)}]$.
  The result then follows by Lemmas~\ref{lem:timeToZeroET} and~\ref{lem:timeToFinishET}.
\end{proof}

\subsection*{Power-Law Mutation}

\begin{proof}[Proof of Lemma~\ref{lem:powerLawTimeToZeroInsideMagicInterval}]
  We consider the smallest distance of the population to $(0, \dots, 0)^\t$, that is, we consider $(X_t)_{t \in \N}$ where, for all $t \in \N$, we have $X_t = \min_{z \in P^{(t)}} \|z\|_1$.
  Note that~$T_1$ is the first point in time such that $X_{T_1} = 0$.
  We aim to apply the variable drift theorem (Theorem~\ref{thm:vardrift}) to~$X$.

  To this end, consider an iteration $t \in \N$ such that $t < T_1$.
  By Lemma~\ref{lem:noIncreaseInL1Norm}, the minimum L1-norm in the population cannot increase.
  Hence, using the notation of Algorithm~\ref{algo:GSEMO}, we consider the case that we choose~$x^{(t)}$ such that $\|x^{(t)}\|_1 = X_t$ and that the offspring $y^{(t)}$ is such that $\|y^{(t)}\|_1 < \|x^{(t)}\|_1$.
  Let~$A$ denote the event that we choose such an $x^{(t)}$.
  By Lemma~\ref{lem:popsize}, the probability of~$A$ is at least $1/(2a + 1)$.

  We first consider the drift in component $i \in [n]$ of~$x^{(t)}$.
  SEMO chooses to mutate (only) this position independently with probability at least~$1/n$, and GSEMO does so with probability at least $(1/n)(1 - 1/n)^{n - 1} \geq 1/(en)$.
  Afterward, with probability~$1/2$, the change by the mutation has the correct sign such that~$x^{(t)}_i$ moves toward~$0$.
  Any change by $k \in [|x^{(t)}_i|]$ results in an improvement by~$k$ in component~$i$.
  Applying Theorem~\ref{thm:sumsToIntegrals}, the expected change in component~$i$ is then at least
  \begin{align}
    \label{eq:powerLawTimeToZeroInsideMagicInterval:driftInComponent}
    \frac{1}{2en} \frac{1}{\zeta(\beta)} \sum_{k = 1}^{|x^{(t)}_i|} k \cdot k^{-\beta}
     & \geq \frac{1}{2en} \frac{1}{\zeta(\beta)} \int_{1}^{|x^{(t)}_i| + 1} k^{1 - \beta} \d k      \\
    \notag
     & = \frac{1}{2en} \frac{1}{\zeta(\beta)} \frac{(|x^{(t)}_i| + 1)^{2 - \beta} - 1}{2 - \beta} .
  \end{align}
  Since the expected progress in the L1-norm is the sum of the progresses over the components, we obtain
  \begin{align*}
     & \E[X_t - X_{t + 1} \mid X_t; A]                                                                                                      \\
     & \qquad\geq \frac{1}{2en} \frac{1}{\zeta(\beta)} \frac{1}{2 - \beta} \sum_{i \in [n]} \bigl((|x^{(t)}_i| + 1)^{2 - \beta} - 1\bigr) .
  \end{align*}
  We derive a lower bound for the right-hand side.
  Since $2 - \beta \in (0, 1)$, the function $z \mapsto z^{2 - \beta} - 1$ is concave, and since we consider a sum of such concave functions, the sum is also concave.
  A concave function takes its minimum at the border of a bounded set.
  Hence, this expression is minimized if a single component has all the mass, that is, if there is an $i \in [n]$ such that $|x^{(t)}_i| = \|x^{(t)}\|_1$.
  Thus,
  \begin{align*}
     & \E[X_t - X_{t + 1} \mid X_t; A]                                                                                       \\
     & \qquad\geq \frac{1}{2en} \frac{1}{\zeta(\beta)} \frac{1}{2 - \beta} \bigl((\|x^{(t)}\|_1 + 1)^{2 - \beta} - 1\bigr) .
  \end{align*}
  We further bound this expression from below by a case distinction.

  If $\|x^{(t)}\|_1 \geq 2^{1/(2 - \beta)} - 1$, then $(\|x^{(t)}\|_1 + 1)^{2 - \beta} \geq 2$, and thus $(\|x^{(t)}\|_1 + 1)^{2 - \beta} - 1 \geq (1/2)(\|x^{(t)}\|_1 + 1)^{2 - \beta}$.
  As $X_t = \|x^{(t)}\|_1$ and $X_t + 1 \geq X_t$, this results in
  \begin{align*}
    \E[X_t - X_{t + 1} \mid X_t; A] \geq \frac{1}{2en} \frac{1}{\zeta(\beta)} \frac{1}{2 - \beta} \frac{1}{2} X_t^{2 - \beta} .
  \end{align*}

  For $X_t \in [1, 2^{1/(2 - \beta)} - 1) \cap \N$, we consider only the sum in equation~\eqref{eq:powerLawTimeToZeroInsideMagicInterval:driftInComponent} for $k = 1$, resulting in
  \begin{align*}
    \E[X_t - X_{t + 1} \mid X_t; A] \geq \frac{1}{2en} \frac{1}{\zeta(\beta)} .
  \end{align*}

  Since the drift is monotonically non-decreasing in~$X_t$, we apply the variable drift theorem (Theorem~\ref{thm:vardrift}).
  To this end, we apply a case distinction with respect to whether $X_t \geq 2^{1/(2 - \beta)} - 1 \eqqcolon \ell$, we account for the probability of~$A$ derived at the beginning of the proof, and we apply Theorem~\ref{thm:sumsToIntegrals}.
  We obtain
  \begin{align*}
     & \E[T_1 \mid x^{(0)}]                                                                                                                    \\
     & \leq (2a + 1)\sum_{k \in [\lceil\ell\rceil - 1]} 2en\zeta(\beta)                                                                        \\
     & \qquad+ (2a + 1)\sum_{k = \lceil\ell\rceil}^{X_0} 4en\zeta(\beta)(2 - \beta) k^{-(2 - \beta)}                                           \\
     & \leq (2a + 1) \cdot 2en\zeta(\beta) \left(\ell + 2(2 - \beta) \int_{\lceil\ell\rceil - 1}^{X_0} k^{-(2 - \beta)} \d k\right)            \\
     & \leq (2a + 1) 2en\zeta(\beta) \left(\ell + 2(2 - \beta) \frac{X_0^{\beta - 1} - (\lceil\ell\rceil - 1)^{\beta - 1}}{\beta - 1}\right) .
  \end{align*}
  Removing the term $-(\lceil\ell\rceil - 1)^{\beta - 1}$ and noting that $X_0 = \|x^{(0)}\|_1$ concludes the proof.
\end{proof}

\begin{proof}[Proof of Lemma~\ref{lem:powerLawTimeToCompletingParetoFront}]
  If $a = 0$, then $T_2 = 0$.
  Hence, we assume in the following that $a \in \N_{\geq 1}$.
  We consider several steps until at least a half of~$F^*$ is covered.
  Afterward, we consider one more stage such that~$F^*$ is fully covered.

  For step $k \in \N$ with $k \leq \lfloor \log_2(2a + 1) \rfloor - 1$, consider a partition of $[-a .. a]$ into $2^{k + 1}$ consecutive intervals of roughly equal size and such that at least $2^k$ of them contain at least one individual in the current population (called \emph{hit}).
  Formally, let $\ell(k) \coloneqq (2a + 1) / 2^{k + 1}$, noting that~$\ell(k)$ is never integer and that $\ell(k) > 1$.
  We consider the partitioning of $[-a .. a]$ into $\bigcup_{i \in [0 .. 2^{k + 1} - 1]} \{[-a + \lceil i \ell(k) \rceil .. -a + \lfloor (i + 1) \ell(k) \rfloor] \cap [-a .. a]\}$.
  Step~$k$ ends once all $2^{k + 1}$ intervals are hit.
  Note that due to the assumption about~$S$, the constraints for step~$1$ are satisfied in iteration~$S$.

  Consider step $k \in \N$.
  Each interval has a size of at most $\lceil \ell(k) \rceil$.
  Since we halve the intervals from the previous step (or the entire interval if $k = 0$), each interval that is not hit yet is neighboring at least one interval that is already hit.
  Hence, in the worst case, in order to create any individual in an unhit interval, a distance of at least $\lceil \ell(k) \rceil$ needs to be overcome from the neighboring interval that contains an individual.
  Each point in an unhit interval (of which there are at least $\lfloor \ell(k) \rfloor$) is a possible target in order to hit the interval.

  The previous discussion implies that in order to hit an unhit interval~$I$, the following sequence of events is sufficient.
  First, choose a parent from a hit neighbor, which has, by Lemma~\ref{lem:popsize}, a probability of at least $1/(2a + 1)$, and then mutate the first component of the parent, which happens for the SEMO with probability at least~$1/n$ and for the GSEMO with probability at least $(1/n)(1 - 1/n)^{n - 1} \geq 1/(en)$.
  Afterward, the mutation needs to adjust the first component into the correct direction, which happens with probability~$1/2$.
  Since all of these decisions are independent, this results in a probability of at least $(1/(2a + 1))(1/(2en))$ for choosing a specific parent and mutating it into a direction such that it can hit~$I$.
  We call this chain of events~$A$.
  Such an individual can be mutated in at least $\lfloor \ell(k) \rfloor$ ways in its first component, needing to hit a distance of at least $\lceil \ell(k) \rceil$.
  Conditional on~$A$, we bound the probability that the mutation results in an offspring that hits~$I$ from below.
  To this end, we use Theorem~\ref{thm:sumsToIntegrals}, that $\lfloor \ell(k) \rfloor = \lceil \ell(k) \rceil - 1$, as~$\ell(k)$ is not integer, and that $2\lceil \ell(k) \rceil - 1 \geq \frac{3}{2}\lceil \ell(k) \rceil$, as $\ell(k) > 1$.
  We obtain
  \begin{align*}
     & \sum_{j \in [0 .. \lfloor \ell(k) \rfloor - 1]} \frac{1}{\zeta(\beta)} \left(\lceil \ell(k) \rceil + j\right)^{-\beta}                            \\
     & \quad\geq \frac{1}{\zeta(\beta)} \int_{0}^{\lfloor \ell(k) \rfloor} \left(\lceil \ell(k) \rceil + j\right)^{-\beta} \d j                          \\
     & \quad= \frac{1}{\zeta(\beta)} \frac{\lceil \ell(k) \rceil^{1 - \beta} - (\lceil \ell(k) \rceil + \lfloor \ell(k) \rfloor)^{1 - \beta}}{\beta - 1} \\
     & \quad= \frac{1}{\zeta(\beta)} \frac{\lceil \ell(k) \rceil^{1 - \beta} - (2\lceil \ell(k) \rceil - 1)^{1 - \beta}}{\beta - 1}                      \\
     & \quad\geq \frac{1}{\zeta(\beta)} \frac{\lceil \ell(k) \rceil^{1 - \beta} - (\frac{3}{2}\lceil \ell(k) \rceil)^{1 - \beta}}{\beta - 1}             \\
     & \quad\geq \frac{1}{\zeta(\beta)} \frac{1}{\beta - 1} \left(1 - \left(\frac{3}{2}\right)^{1 - \beta}\right) 2^{1 - \beta} \ell(k)^{1 - \beta} .
  \end{align*}
  Overall, accounting for the probability of~$A$, the probability that a mutation covers~$I$ is at least
  \begin{align*}
    \frac{1}{2a + 1} \frac{1}{2en} \frac{1}{\zeta(\beta)} \frac{1}{\beta - 1} \left(1 - \left(\frac{3}{2}\right)^{1 - \beta}\right) 2^{1 - \beta} \ell(k)^{1 - \beta} \\
    \hfill\eqqcolon q(k) .
  \end{align*}

  The probability that interval~$I$ is not covered within $s(k) \coloneqq \ln(2^{k + 1})/q(k)$ iterations is at most $(1 - q(k))^{s(k)} \leq \exp\bigl(-s(k)q(k)\bigr) \leq 2^{-(k + 1)}$.
  Hence, the probability to cover all~$2^k$ uncovered intervals within~$s$ iterations is, by Bernoulli's inequality, at least $(1 - 2^{-(k + 1)})^{2^{k}} \geq 1/2$.
  By a restart argument, step~$k$ lasts in expectation at most $2s(k)$ iterations.

  Summing over all steps from~$0$ to $\lfloor \log_2(2a + 1) \rfloor - 1$, we get that the total expected length of all steps.
  To this end, we use the notation $q' \coloneqq q(k) / \ell(k)^{1 - \beta}$ as well as Lemma~\ref{lem:specialGeometricSeries}.
  We an upper bound of
  \begin{align*}
     & 2 \cdot \sum_{k = 0}^{\lfloor \log_2(2a + 1) \rfloor - 1} s(k)                                                            \\
     & \quad=2 \cdot \sum_{k = 0}^{\lfloor \log_2(2a + 1) \rfloor - 1} \frac{\ln(2^{k + 1})}{q(k)}                               \\
     & \quad\leq \frac{2 \ln(2)}{q'} \cdot \sum_{k = 0}^{\lfloor \log_2(2a + 1) \rfloor - 1} \frac{k + 1}{\ell(k)^{1 - \beta}}   \\
     & \quad\leq \frac{2 \ln(2)}{q'} (2a + 1)^{\beta - 1} \cdot \sum_{k = 0}^{\infty} (k + 1) \left(2^{1 - \beta}\right)^{k + 1} \\
     & \quad= \frac{2 \ln(2)}{q'} (2a + 1)^{\beta - 1} \cdot \frac{2^{1 - \beta}}{\left(1 - 2^{1 - \beta}\right)^2} .
  \end{align*}
  Afterward, at least half of~$F^*$ is covered.

  We continue with the last stage.
  Let $i \in [1, |F^*|/2] \cap \N \subseteq [a + 1]$ denote the number of solutions from~$F^*$ that are still not in the population.
  Note that by the previous sequence of steps, each such solution not in the population has at least one neighbor in the population in distance~$1$.
  Hence, the probability to add a new solution to the population is at least $i/((2a + 1) 2en \zeta(\beta))$.
  The expected time until a new solution is added is the multiplicative inverse of this probability.
  Hence, the total expected time of the last stage is bounded from above by
  \begin{align*}
     & (2a + 1) \cdot 2en \zeta(\beta) \cdot \sum_{i \in [a + 1]} \frac{1}{i}  \\
     & \qquad\leq (2a + 1) \cdot 2en \zeta(\beta) \bigl(\ln(a + 1) + 1\bigr) .
  \end{align*}
  The proof is concluded by noting that $\E[T_2 \mid S, x^{(0)}]$ is bounded from above by the estimate of the two sums above.
\end{proof}

\begin{proof}[Proof of Theorem~\ref{thm:powerLawRunTimeSemoGsemo}]
  For either algorithm, let~$T_1$ and~$T_2$ be defined as in Lemmas~\ref{lem:powerLawTimeToZeroInsideMagicInterval} and~\ref{lem:powerLawTimeToCompletingParetoFront}, respectively, letting~$S$ from Lemma~\ref{lem:powerLawTimeToCompletingParetoFront} be~$T_1$.
  Then it follows that $T = T_1 + T_2$.
  By linearity of expectation and by the law of total probability, we get $\E[T \mid x^{(0)}] = \E[T_1 \mid x^{(0)}] + \E[\E[T_2 \mid T_1, x^{(0)}] \mid x^{(0)}]$.
  The result then follows by Lemmas~\ref{lem:powerLawTimeToZeroInsideMagicInterval} and~\ref{lem:powerLawTimeToCompletingParetoFront}.
\end{proof}

\section*{Empirical Analysis}

We follow the same experimental setup as in the main paper but for the values of $n \in \{4, 10\}$.
For the initial solution, we choose $x^{(0)}_2 = 100a$ and all of the other components of~$x^{(0)}$ as~$0$.

\paragraph{First scenario.}
Our results are depicted in Table~\ref{tab:means+sd_n=4_10}.
The results are qualitatively comparable to those in Table~\ref{tab:means+sd_n=2}.
That is, the step size of $1/q \in \{20, 50\}$ is best in terms of total runtime for the exponential-tail mutation.
For $n = 10$, the relative difference between $1/q = 20$ and $1/q = 50$ is very small, which is why we choose for consistency reasons with $n \in \{2, 4\}$ also $1/q = a/4$ for $n = 10$ for the second scenario.

The runtime of the power-law mutation for either of the two phases we consider is better than any of those of the exponential-tail mutation.
Unit-step mutation performs by far the worst.

\paragraph{Second scenario.}
Our results are depicted in Figure~\ref{fig:additional-plots}.
As in our results for $n = 2$ (Figure~\ref{fig:a-plot}), unit-step mutation performs by far the worst, with exponential-tail and power-law mutation performing both roughly linearly.
The separation between exponential-tail mutation and power-law mutation becomes more distinct with increasing values of~$n$, to the extent where the standard deviations do not intersect anymore for $n = 10$.
This indicates a clear advantage of the power-law mutation for increasing values of~$n$.

\begin{table*}
  \caption{
    Means and standard deviations in percent of 1st hitting time (phase~$1$), Pareto set cover time (phase~$2$), and total runtime (number of function evaluations) for scenario~$1$ for the GSEMO optimizing~$f$ with three different mutation operators: unit-step (U), exponential-tail (E), and power-law (P).
    The column $1/q$ refers to the \emph{step size} chosen for the parameter~$q$ of E.
    For~P, we chose $\beta = \frac{3}{2}$.
    The experiments were run with $a = 200$ and $x^{(0)}_2 = 100 a$ and all other components of~$x^{(0)}$ being~$0$, with~$50$ independent runs per row for $n \in \{4, 10\}$.
  }
  \label{tab:means+sd_n=4_10}
  \begin{minipage}{0.48\linewidth}
    \caption*{Results for $n = 4$.}
    \begin{tabular}{cr*{3}{r@{$\pm$}r}}
                                    & $1/q$ & \multicolumn{2}{r}{1st hit} & \multicolumn{2}{r}{cover} & \multicolumn{2}{r}{total}                         \\
      \toprule
      U                             &       & 850\,395                    & 31                        & 995\,144                  & 34 & 1\,845\,539 & 13 \\ \hdashline
      \vphantom{\rule{0 pt}{1 em}}E & 5     & 164\,192                    & 11                        & 607\,63                   & 33 & 224\,955    & 10 \\
                                    & 10    & 56\,364                     & 10                        & 43\,821                   & 22 & 100\,185    & 10 \\
                                    & 20    & 21\,782                     & 10                        & 38\,192                   & 15 & 59\,974     & 11 \\
                                    & 50    & 10\,701                     & 29                        & 40\,518                   & 17 & 51\,219     & 16 \\
                                    & 100   & 13\,813                     & 47                        & 48\,918                   & 17 & 62\,731     & 18 \\
                                    & 200   & 21\,745                     & 50                        & 65\,441                   & 18 & 87\,186     & 23 \\
                                    & 500   & 48\,866                     & 54                        & 113\,862                  & 17 & 162\,728    & 22 \\ \hdashline
      \vphantom{\rule{0 pt}{1 em}}P &       & 2\,678                      & 38                        & 34\,075                   & 18 & 36\,753     & 17 \\

      \bottomrule
    \end{tabular}
  \end{minipage}
  \begin{minipage}{0.48\linewidth}
    \caption*{Results for $n = 10$.}
    \begin{tabular}{cr*{3}{r@{$\pm$}r}}
                                    & $1/q$ & \multicolumn{2}{r}{1st hit} & \multicolumn{2}{r}{cover} & \multicolumn{2}{r}{total}                         \\
      \toprule
      U                             &       & 1\,792\,117                 & 31                        & 2\,467\,353               & 36 & 4\,259\,470 & 12 \\ \hdashline
      \vphantom{\rule{0 pt}{1 em}}E & 5     & 458\,488                    & 8                         & 167\,105                  & 32 & 625\,593    & 10 \\
                                    & 10    & 162\,492                    & 8                         & 113\,735                  & 20 & 276\,227    & 9  \\
                                    & 20    & 77\,547                     & 12                        & 108\,167                  & 16 & 185\,715    & 9  \\
                                    & 50    & 74\,820                     & 18                        & 123\,049                  & 22 & 197\,869    & 17 \\
                                    & 100   & 113\,681                    & 21                        & 139\,902                  & 17 & 253\,583    & 13 \\
                                    & 200   & 186\,919                    & 25                        & 183\,954                  & 15 & 370\,872    & 16 \\
                                    & 500   & 379\,859                    & 32                        & 321\,510                  & 11 & 701\,369    & 18 \\ \hdashline
      \vphantom{\rule{0 pt}{1 em}}P &       & 8\,516                      & 35                        & 93\,739                   & 17 & 102\,255    & 17 \\
      \bottomrule
    \end{tabular}
  \end{minipage}
\end{table*}

\begin{figure*}
  \begin{subfigure}{0.48\linewidth}
    \includegraphics[width=\columnwidth]{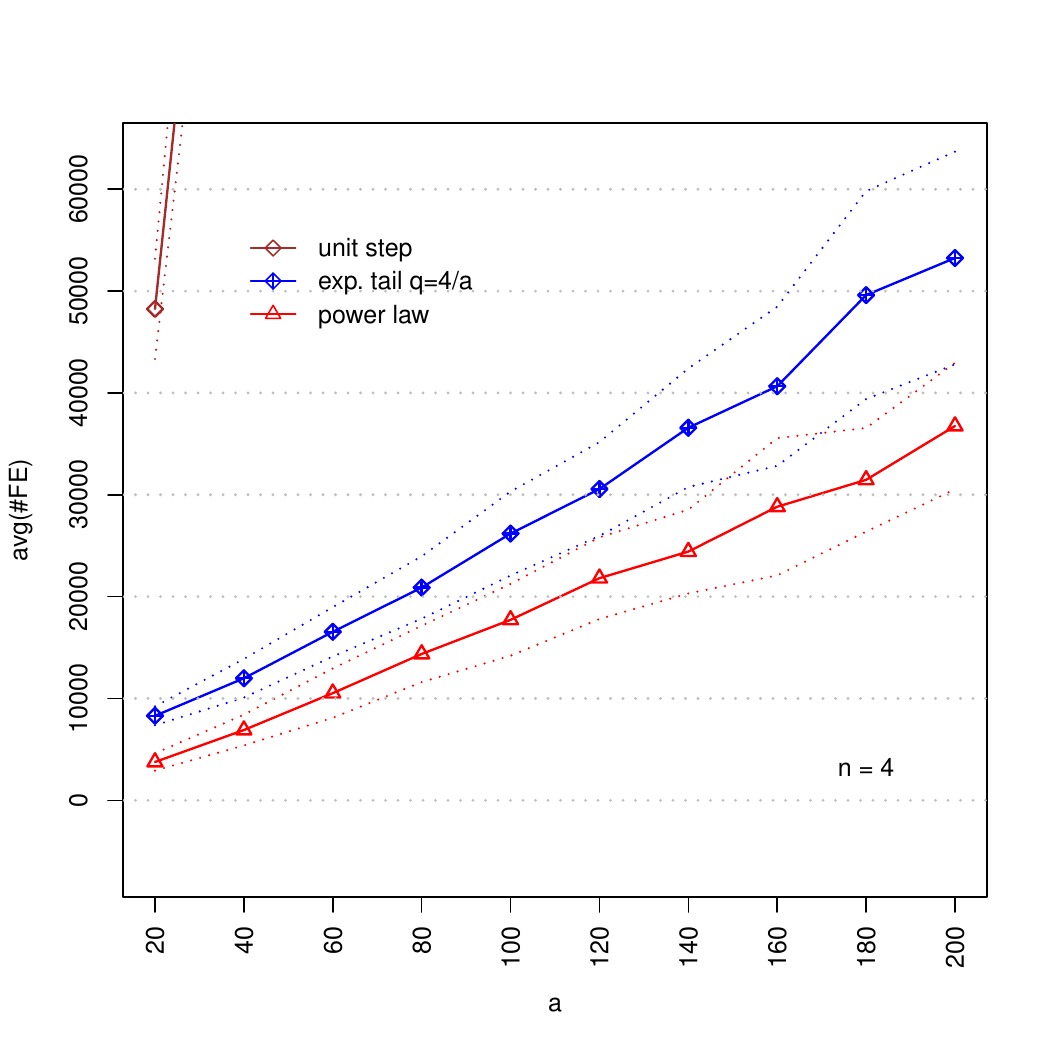}
    \caption{
      Scenario~$2$ for $n = 4$.
    }\label{fig:a-plot-n=4}
  \end{subfigure}
  \hfill
  \begin{subfigure}{0.48\linewidth}
    \includegraphics[width=\columnwidth]{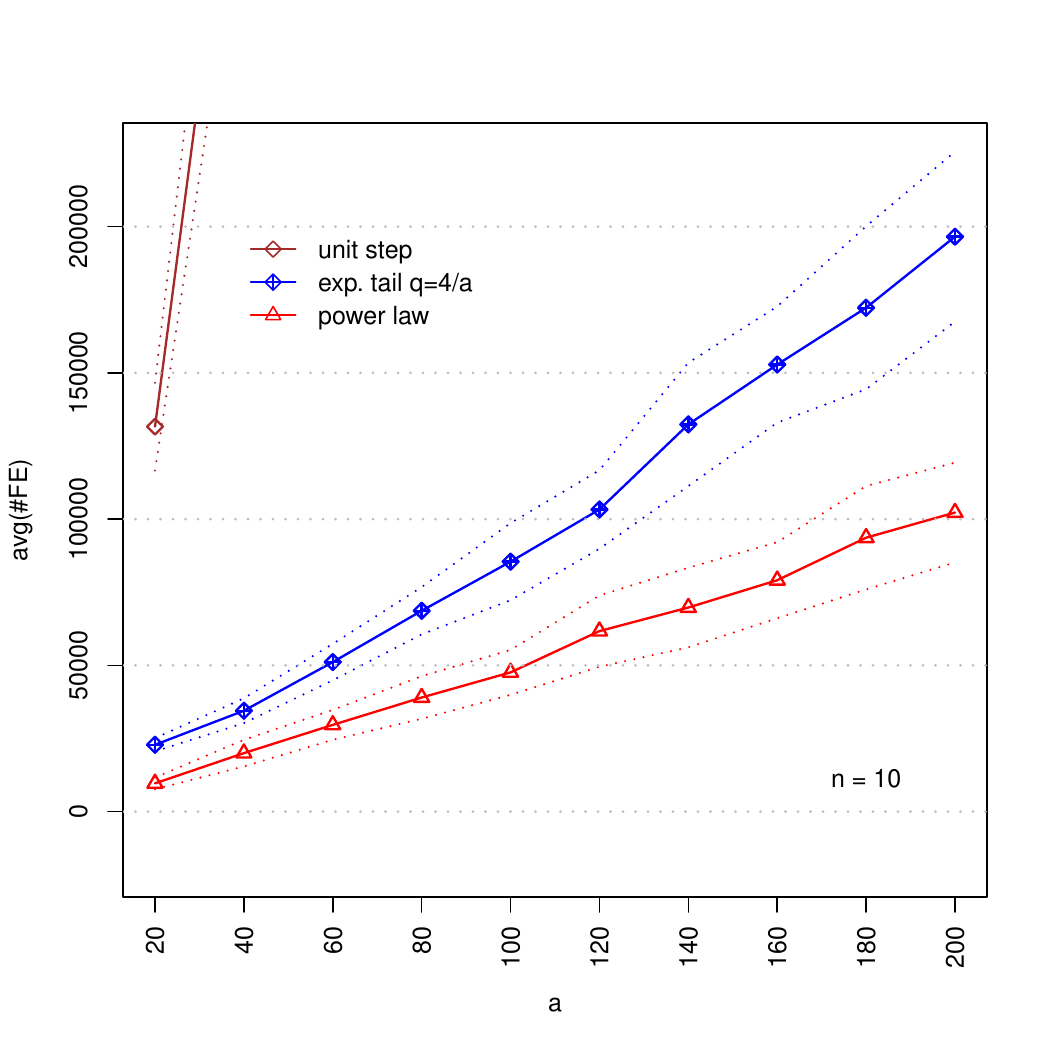}
    \caption{
      Scenario~$2$ for $n = 10$.
    }\label{fig:a-plot-n=10}
  \end{subfigure}
  \caption{
    The results from scenario~$2$ for $n = 4$ (left) and $n = 10$ (right).
    Average function evaluations for varying~$a$ for the GSEMO optimizing~$f$ with three different mutation operators:
    unit-step (brown diamonds), exponential-tail (blue diamonds) with $\frac{1}{q} = \frac{a}{4}$, and power-law (red triangles) with $\beta = \frac{3}{2}$.
    For each point, $50$ independent runs were conducted, choosing~$x^{(0)}_2 = 100a$ and all other components of~$x^{(0)}$ as~$0$.
    The dotted lines depict the standard deviation of each curve.
  }
  \label{fig:additional-plots}
\end{figure*}

\end{document}